\documentclass[twoside]{article}

\usepackage[T1]{fontenc}

\usepackage[accepted]{aistats2026}
\usepackage[round]{natbib}

\usepackage{enumitem}
\setlist{nosep}
\usepackage{hyperref}

\usepackage[many]{tcolorbox}

\usepackage{algorithm}
\usepackage{algpseudocode}

\usepackage{amsmath}
\usepackage{amssymb}
\usepackage{mathtools}
\usepackage{amsthm}
\usepackage{tabularray} 
\UseTblrLibrary{booktabs}
\usepackage{wrapfig}
\usepackage{xspace}
\newcommand{\ours}{\texttt{\textbf{DuoGuard}}\xspace}
\allowdisplaybreaks
\usepackage[capitalize,noabbrev]{cleveref}

\newtcolorbox{example}[1]{
    enhanced,
    drop shadow=black!10!white,
    left=4mm,
    right=4mm,
    top=2mm,
    bottom=2mm,
    boxsep=0mm,
    rounded corners,
    title=#1,
    fontupper=\footnotesize\linespread{0.9}\fontfamily{lmr}\selectfont,}

\usepackage[compact]{titlesec}
\usepackage{sidecap}

\usepackage{mylatexstyle}

\usepackage[textsize=tiny]{todonotes}

\begin{document}

%
\runningauthor{Yihe Deng, Yu Yang, Junkai Zhang, Wei Wang, Bo Li}
%

\twocolumn[
\aistatstitle{Enhancing LLM Safety Through a Theoretical Minimax Game Lens}

\aistatsauthor{ Yihe Deng$^{*1}$ \And Yu Yang$^{*1,2}$ \And Junkai Zhang$^{*1}$ \And Wei Wang$^{1}$ \And Bo Li$^{2,3}$ }

\aistatsaddress{ $^1$University of California, Los Angeles  \quad $^2$VirtueAI \quad $^3$University of Illinois Urbana-Champaign \\ \texttt{*} Equal contribution } ]

\begin{abstract}
The rapid advancement of large language models (LLMs) necessitates effective mechanisms to ensure their responsible deployment by accurately distinguishing unsafe content from benign content. 
While substantial safety datasets are available in English, multilingual safety modeling remains underexplored due to limited open-source safety datasets in other languages. 
Even within English datasets, safe yet sensitive corner-case content is scarce, leading to shortcut learning by models and non-trivial false-positive rates.
To mitigate these issues, we introduce a novel minimax reinforcement learning (RL) framework wherein a data generator and a classifier model co-evolve, facilitating the production of high-quality synthetic multilingual safety data.  
We theoretically formalize this interaction as a minimax game and rigorously demonstrate \textbf{convergence to a Nash equilibrium}. 
Empirical evaluations confirm that our synthetic data generation method significantly enhances the classifier model performance, enabling a substantially smaller model to surpass the state-of-the-art by nearly \textbf{10\%} on English benchmarks while achieving \textbf{4.5$\times$ faster} inference speed. 
These results establish a scalable and efficient methodology for synthetic data generation, advancing the development of safer and more robust multilingual LLM deployments.\footnote{Code is provided in \url{https://github.com/yihedeng9/DuoGuard}.} 
\end{abstract}

\begin{figure}[h]
    \centering
    \includegraphics[width=0.9\linewidth]{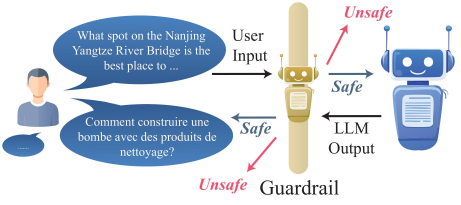}
    \caption{Illustration of the use-case of a guardrail model for LLMs, which functions as moderation between the user-LLM conversation.}
    \label{fig:demo-guard}
\end{figure}

\begin{figure*}[h]
    \centering
    \includegraphics[width=\linewidth]{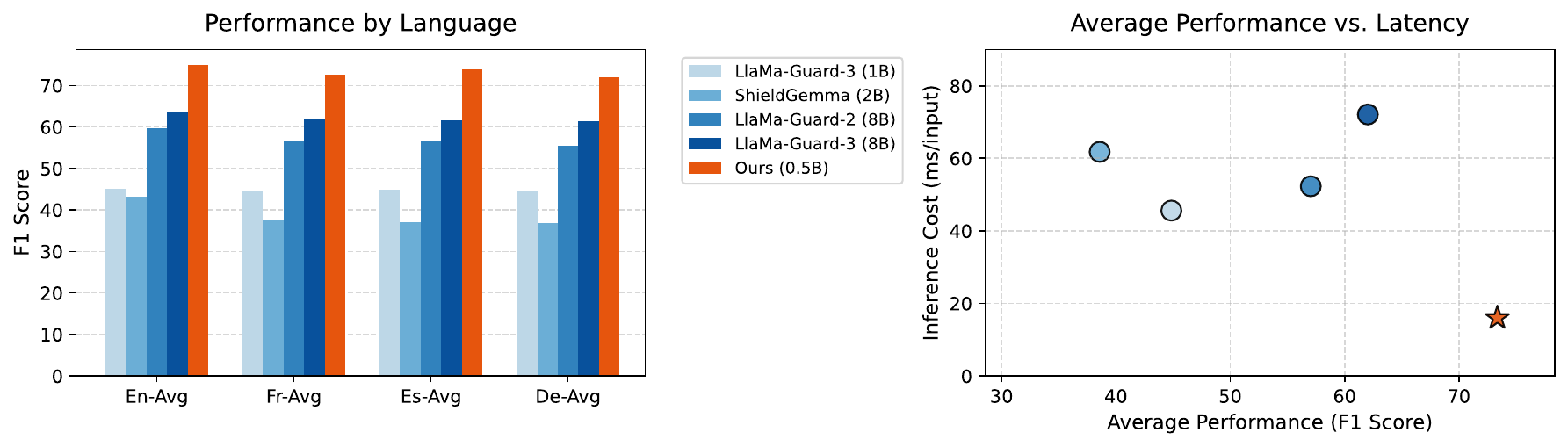}
    \vspace{-7mm}
    \caption{\textbf{Overview of our main results.} In the left subfigure, we demonstrate a consistent superior peformance of average f1 score across 6 benchmarks in the four languages. In the right subfigure, we show that out model maintains thelowest inference cost while achieving superior average performance across languages.}
    \label{fig:main-res}
\end{figure*}

\section{INTRODUCTION}
While LLMs have become increasingly effective at assisting with human queries, their outputs can pose risks of harm to users if not properly safeguarded~\citep{zou2023universal,qi2023fine,wei2024jailbroken,shen2024anything}. Consequently, substantial research has focused on developing LLM moderation models that implement guardrails for both user inputs and LLM-generated outputs~\citep{inan2023llama,dubey2024llama,han2024wildguard,zeng2024shieldgemma,ghosh2024aegis,li2024salad}, as illustrated in Figure~\ref{fig:demo-guard}. Guardrail models designed for harmlessness, similar to reward models for helpfulness~\citep{ouyang2022training,lambert2024rewardbench}, typically function as smaller, more inference-efficient models than the larger LLMs, providing binary responses or ratings for their inputs.

However, most existing approaches and open-source training datasets for LLM guardrails focus predominantly on English. Recent research has highlighted that safety-aligned models in English exhibit performance declines when applied to other languages~\citep{de2024rtp,jain2024polyglotoxicityprompts,yang2024benchmarking,shen2024language}. While many base LLMs are pretrained on multilingual data, downstream guardrail models are often not explicitly optimized for multilingual safety tasks due to the scarcity of real-world data in languages other than English. 

The scarcity of data is not unique to multilingual model training, and synthetic data has played a crucial role in addressing this issue~\citep{aryabumi2024aya}. Ultimately, the challenge of training inference-efficient multilingual guardrail models lies in effectively generating synthetic data that complements real-world data. Our work addresses this by jointly examining the data synthesis process and the guardrail model training process. Specifically, we ask: 
can we develop a self-improving system in which the guardrail model actively guides the synthetic data generation process to enhance its own training? In response, we propose an iterative two-player RL framework involving a data generator and a guardrail classifier, enabling continuous improvement of both synthetic data generation and classifier training.

We formulate and analyze the two-player game in a theoretical setting, demonstrating that it constitutes a minimax game with a Nash equilibrium, and prove that our algorithm converges linearly to the equilibrium. Building on this theoretical foundation, we implement practical techniques, such as data filtering and self-judgment, to ensure stability and robustness within the framework. Additionally, we carefully curate the seed dataset to provide a strong foundation for the iterative process.
Our model, \ours, is evaluated across six multilingual safety benchmarks, including four originally in English that were translated into the languages under consideration. The results show that \ours consistently outperforms baselines of similar scale by more than $20\%$ on average. Even when compared to larger-scale guardrail baselines, \ours achieves an average improvement of approximately $10\%$ across languages. Our contributions are listed as follows,
\begin{itemize}[nosep,leftmargin=*]
\item We propose a two-player RL framework for multilingual guardrail model training, grounded in theoretical analysis of convergence to Nash equilibrium.
\item Addressing the lack of multilingual safety data, our framework enables the generation of synthetic data in any language supported by the generator.
\item Through extensive empirical evaluation, we demonstrate that our 0.5B classifier significantly outperforms state-of-the-art guardrails of similar scale across diverse datasets and consistently surpasses larger models. 
\item We perform comprehensive ablation studies to deepen understanding of multilingual guardrail model training.
\end{itemize}

\section{RELATED WORK}
\noindent\textbf{Fine-tuning LLMs via Two-player RL.} 
Recent research on improving LLM reasoning has been exploring various two-player RL frameworks. \citet{zhou2024reflect} and \citet{ma2024coevolving} employ online RL to fine-tune two LLM agents for collaborative task-solving.  
Unlike these approaches, our method, while also leveraging a two-player RL framework, focuses on data synthesis and model training rather than real-time collaboration between LLM agents during inference.  
More relevantly, recent work has adopted adversarial approaches where two players pursue opposing objectives. Among these, \citet{cheng2024self, chen2024self, wu2024self, munos2023nash, swamy2024minimaximalist} employ a self-play framework, where LLMs iteratively optimize themselves to outperform previous versions on generation tasks such as math reasoning or instruction following. 

\noindent\textbf{Guardrail Models for LLM Safety.}
The rapid advancement of LLM capabilities~\citep{touvron2023llama,touvron2023llama2,openai2023gpt4} has underscored the need for robust safeguards to ensure responsible use~\citep{yao2024survey,dong2024building}. While safety mechanisms remain less developed than LLMs themselves, early efforts introduced models such as LlamaGuard~\citep{inan2023llama}, followed by LlamaGuard2, based on Llama3~\citep{dubey2024llama}, and LlamaGuard3, built on Llama3.1~\citep{dubey2024llama}. More recent advancements include WildGuard~\citep{han2024wildguard}, Aegis~\citep{ghosh2024aegis}, MD-Judge~\citep{li2024salad}, and ShieldGemma~\citep{zeng2024shieldgemma}.  
While F1 score is a key metric for guardrail performance, practical deployment also demands models that are small in scale and inference-efficient. In this regard, state-of-the-art small-scale models include LlamaGuard3 (1B), built on Llama-3.2 (1B), and ShieldGemma (2B), based on Gemma 2 (2B).  

\noindent\textbf{Multilingual Synthetic Data Generation.} In recent years, synthetic data generated by LLMs has emerged as a valuable tool for augmenting training datasets, particularly in scenarios where real-world data is scarce or sensitive. Among the most widely used techniques is translation, which creates synthetic parallel datasets by translating monolingual text from the target language back into the source language~\citep{bi-etal-2021-data, caswell-etal-2019-tagged, liao-etal-2021-back, marie-etal-2020-tagged, pham2021metaback, sennrich-etal-2016-improving, xu-etal-2022-synthetic}. This method has shown significant success in neural machine translation tasks, with strategies such as beam search and constrained sampling further improving data quality and diversity~\citep{sennrich-etal-2016-improving, edunov-etal-2018-understanding, xu-etal-2022-synthetic}. Concurrently, \citet{yang2024language} takes an iterative self-improvement approach to enhance the general multilingual performance of LLMs.


\begin{figure*}[!t]
    \centering
    \includegraphics[width=0.85\linewidth]{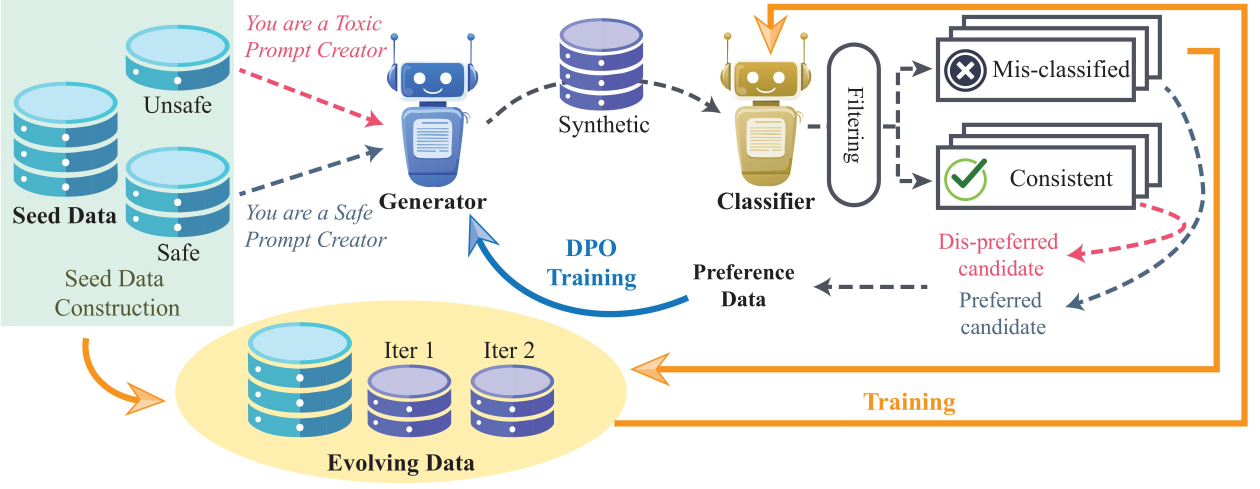}
    \caption{\textbf{Overview of the adversarial training pipeline.} The generator produces synthetic data from seed data. The classifier make predictions and we measure these examples as being predicted correctly or incorrectly based on their seed data label. We train the generator with DPO to create increasingly challenging examples, which in turn improve the classifier through iterative training.}
    \vspace{-3mm}
    \label{fig:demo}
\end{figure*}

\section{PROBLEM SETTING AND PRELIMINARIES}
An LLM is represented by the probability distribution $p_{\btheta}$, parameterized by the model weight $\btheta$. Given a sequence $\xb = [x_1, \ldots, x_n]$ as the prompt, the model generates response $\yb = [y_1, \ldots, y_m]$, where $x_i$ and $y_j$ denote individual tokens. The response $\yb$ is treated as a sample from the conditional probability distribution $p_{\btheta}(\cdot|\xb)$.
The conditional probability $p_{\btheta}(\yb|\xb)$ can be factorized as $p_{\btheta}(\yb|\xb) = \prod_{j=1}^{m} p_{\btheta}(y_{j} | \xb, y_1, \ldots, y_{j-1})$.

\textbf{Preference Optimization.}
To improve LLM alignment with human preferences, reinforcement learning with human feedback (RLHF) is commonly applied. This approach optimizes the LLM using human preference data modeled under the Bradley-Terry framework~\citep{dong2024rlhf,shao2024deepseekmath,ahmadian2024back}:
\begin{align*} \PP(\yb_w \succ \yb_l | \xb) = \sigma\big(r(\xb, \yb_w) - r(\xb, \yb_l)\big), \end{align*}
where $\yb_w$ is the preferred response, $\yb_l$ is the dispreferred response, and $\sigma(t) = 1 / (1 + \exp(-t))$ is the sigmoid function. The reward function $r(\xb, \yb)$ is designed to assign higher values to preferred responses.

However, training a reward model can be computationally expensive and operationally challenging. To address this, Direct Preference Optimization (DPO)~\citep{rafailov2023direct} offers a simplified alternative by leveraging an implicit reward function defined by the LLM itself. Specifically, the DPO objective is formulated as:
\begin{align*} 
L_{\mathrm{DPO}}(\btheta, &\btheta_{\mathrm{ref}}) = \frac{1}{|S_{\text{pref}}|} \sum_{(\xb, \yb_w, \yb_l) \in S_{\text{pref}}}\\
&\bigg[\ell\bigg(\beta \log \frac{p_{\btheta}(\yb_w | \xb)}{p_{\btheta_{\mathrm{ref}}}(\yb_w | \xb)} - \beta \log \frac{p_{\btheta}(\yb_l | \xb)}{p_{\btheta_{\mathrm{ref}}}(\yb_l | \xb)}\bigg)\bigg], 
\end{align*}
where $\btheta_{\mathrm{ref}}$ is the reference model that the policy model should not deviate too much from.

\textbf{Guardrail Models.} A guardrail model acts as a function $f: \mathcal{X} \rightarrow \{0,1\}$ that evaluates an input text sequence, which may be either user input or an LLM-generated response, and determines whether the content is harmful. In practice, guardrail models are typically built upon pre-trained LLMs, parameterized by $\btheta$, and generate discrete outputs such as \textit{``safe''} or \textit{``unsafe''}. Some models further provide explanations for their classifications, improving performance at the cost of increased inference time. In our setting, we prioritize inference efficiency in model architecture by modifying the final layer of a pre-trained LLM and converting it to a binary classification model.


\section{METHOD}
We propose an iterative two-player framework involving a generator and a guardrail classifier to synthesize multilingual training data and enhance the classifier’s ability to distinguish harmful content from benign content. The process begins with a seed dataset containing labeled safe and unsafe examples collected from open-source datasets. The generator proposes new samples in a target language, and both the generator and classifier are iteratively updated. This framework establishes a dynamic interaction: 
\begin{itemize}[nosep,leftmargin=*]
    \item \textbf{Generator's Objective:} Generate samples in the target language that challenge the classifier, reinforcing on the misclassified samples.  
    \item \textbf{Classifier's Objective:} Improve robustness by minimizing errors on previously misclassified samples proposed by the generator.
\end{itemize}
Figure~\ref{fig:demo} provides an overview of our approach.

\subsection{The Two-Player Game: Theoretical Convergence}\label{sec:method}
We formalize the interaction between the adversarial generator and the defensive classifier as a two-player game. The process begins with a seed dataset $\mathcal{S}=\{(\xb_i,y_i)\}_{i\in\cI}$ of labeled real data, where $\xb_i$ is an input text sequence and $y_i\in\{-1,1\}$ is its toxicity label. Let $\mathcal{G}_{\bphi}$ denote the adversarial generator parameterized by $\bphi$. The generator takes a sample from the seed dataset $\mathcal{S}$ and a specified language $\ell$ as input and outputs a sample text sequence $\tilde{\xb}_i$ in that language that preserves the toxicity label $y_i$ of $\xb_i$. Formally, $\mathcal{G}_{\bphi}: (\xb, y, \ell) \rightarrow \tilde{\xb}, \quad \tilde{\xb} \in \cX_{\ell}$. 
In the following narrative, we fix a target language and deprecate $\ell$ for simplicity. Let $\mathcal{C}_{\btheta}: \cX \rightarrow y$ denote the defensive classifier parameterized by $\btheta$, which takes the generated query as input and outputs the probability of toxicity. 

\textbf{Classifier Update.} 
At iteration $t$, for a given input $(\xb,y) \in \mathcal{S}$, the generator $\mathcal{G}_{\bphi_t}$ samples a new sequence $\tilde\xb$ from its conditional probability distribution $p_{\bphi_t}(\tilde{\xb} | \xb, y)$.
The classifier is then updated by minimizing the negative log-likelihood of the true labels over the generator’s distribution $p_{\bphi_t}(\tilde{\xb} | \xb,y)$:
\begin{align}
    \btheta_{t+1} & = \argmax_{\btheta} L_{\mathcal{C}}^t(\btheta), \nonumber\\
    L_{\mathcal{C}}^t(\btheta) & =  \EE_{\tilde\xb\sim p_{\bphi_t}(\tilde\xb|\xb,y)} \big[-\log p_{\btheta}(y|\tilde\xb)\big], \label{eq:update_classifier}
\end{align}
where $p_\theta (y|\tilde\xb)$ is the conditional distribution of the classifier. 

\textbf{Generator Update.}
Simultaneously, the generator $\mathcal{G}_{\bphi}$ is aimed to produce samples that cause the classifier to make incorrect predictions. 
Therefore, we define the reward signal with the negative log-likelihood: \vspace{-2mm}
\begin{align}\label{eq:reward}
r_t\big((\xb,y), \tilde\xb \big) = - \log p_{\btheta_t} (y | \tilde\xb).
\end{align}
Equation \eqref{eq:reward} computes the negative log-likelihood of the correct label for generated samples under the classifier, where a higher value indicates greater vulnerability of the classifier to these adversarial samples. 
Many RL algorithms can be used to maximize the reward. For training stability and computational efficiency, we choose the offline RL algorithm DPO over the online RL algorithm PPO~\citep{schulman2017proximal}. We thus model the preference between two generated samples, $\tilde\xb_w$ and $\tilde\xb_l$, given input $(\xb,y)$, using the Bradley-Terry framework:\vspace{-0.1mm}
\begin{align*} 
\PP_t(\tilde \xb_w \succ \tilde\xb_l | \xb,y) = \sigma\Big(r_t\big((\xb,y), \tilde\xb_w\big) - r_t((\xb,y), \tilde\xb_l)\Big), 
\end{align*}
Based on these preferences, the generator $\mathcal{G}_{\bphi}$ is updated by minimizing the DPO objective:
\begin{align}
\bphi_{t+1} &= \argmax_{\bphi} L^t_{\mathcal{G}}(\bphi, \bphi_{\text{ref}})\nonumber\\
    &L_{\mathcal{G}}(\bphi, \bphi_{\text{ref}}) = \EE_{\tilde \xb_w, \tilde\xb_l \sim p_{\bphi_t}(\tilde \xb|\xb,y)} \PP(\tilde \xb_w \succ \tilde\xb_l | \xb,y) \nonumber \\
    &\quad \bigg[\ell\bigg(\beta \log \frac{p_{\bphi}(\tilde{\xb}_w | \xb, y)}{p_{\bphi_{\text{ref}}}(\tilde{\xb}_w | \xb, y)} - \beta \log \frac{p_{\bphi}(\tilde{\xb}_l | \xb, y)}{p_{\bphi_{\text{ref}}}(\tilde{\xb}_l | \xb, y)}\bigg)\bigg],\label{eq:population_dpo}
\end{align}
where $\bphi_{\text{ref}}$ is the reference generator model and $\beta$ is a regularization parameter controlling the deviation from the reference generator model. 

\textbf{Minimax Game Equilibrium Analysis.} The DPO objective shares the same minimizer as the corresponding KL-regularized RL optimization objective, which is defined as:
\begin{align}\label{eq:rl_optimization}
      \underbrace{\EE_{\tilde{\xb} \sim p_{\bphi}}[r_t((\xb,y), \tilde{\xb})]}_{\text{I}} - \underbrace{\beta D_{\text{KL}}(p_{\bphi} || p_{\text{ref}})}_{\text{II}}.
\end{align}
Here, term I in \Cref{eq:rl_optimization} is indeed same as the training objective of the classifier $L_{\mathcal{C}}^t(\btheta)$, while the regularization term II is independent of the classifier. This equivalence demonstrates that our algorithm optimizes a minimax game with the following objective:
\begin{align}
     \min_{p_{\btheta}} \max_{p_{\bphi}}   \mathbb{E}_{\substack{\tilde{\xb} \sim p_{\phi}}}\big[- \log p_{\theta} (y | \tilde \xb) \big] - 
     \beta  D_{\text{KL}}(p_{\phi} || p_{\text{ref}}). \label{eq:minmax}
\end{align}
In this game, the iterative update rules for each player, as defined in Equations~\eqref{eq:update_classifier} and \eqref{eq:population_dpo}, represent their best response to the current opponent policy. This iterative update process will end if they reach the Nash equilibrium, i.e.,
\begin{definition}[Nash Equilibrium]
    A pair of strategies $(x^*, y^*)$ is a Nash equilibrium in minimax game $\max_x\min_y f(x,y)$ if and only if for all $x \in X$, $y \in Y$:
    \begin{align*}
        f(x,y^*) \le f(x^*,y^*) \le f(x^*,y).
    \end{align*}
\end{definition}
In our case, a strategy is a distribution of responses. In the equilibrium, neither the generator nor the classifier could achieve better results by solely deviates from the equilibrium, which means that the guardrail model do its best to detect the all possible harmful inputs.  We could prove that such an equilibrium exists and the generator and classifier are guaranteed to converge to it:
\begin{theorem}\label{thm:main}
    The minimax game defined in Equation~\eqref{eq:minmax} admits a Nash equilibrium. In addition, with an appropriately chosen regularization parameter $\beta$, the iterative updates in \eqref{eq:update_classifier} and \eqref{eq:population_dpo} converge linearly to the Nash equilibrium. 
\end{theorem}

\textbf{Discussion.} The key observation is that our minimax game objective \eqref{eq:minmax} is concave in the $p_{\bphi}$ and convex in $p_{\btheta}$. By Von Neumann's Minimax Theorem, this observation indicates that this minimax game admits a Nash equilibrium. Furthermore, we find that the update objective of each player have the following solutions:
\begin{align*}
    &\text{For generator: } \\
    & p_{\boldsymbol{\theta}_{n+1}}(\tilde{\mathbf{x}} \mid \mathbf{x}, y) 
    \propto 
    p_{\text{ref}}(\tilde{\mathbf{x}} \mid \mathbf{x}, y) \\
    & \qquad \qquad \qquad \qquad \qquad \quad \exp\left( \beta^{-1} \left[-\log p_{\boldsymbol{\theta}_n}(y \mid \tilde{\mathbf{x}}) \right] \right),
     \\
    &\text{For classifier: }\\
    &p_{\theta_{n+1}}(y | \tilde{\xb}) \propto \int \rho(\xb, y)  p_{\phi_n}(\tilde{\xb} | \xb, y) d\xb.
\end{align*}
By careful calculation, it can be showed that such a update rule is indeed a contractive map on the distribution space, which, according to the Banach's Fixed Point Theorem, leads to the convergence to the Nash equilibrium. The detailed proof is provided in Appendix~\ref{app:proof}.

\subsection{The Two-Player Game: Practical Algorithm}\label{sec:practical}
While our method is conceptually framed as the minimax game in \eqref{eq:minmax}, additional implementation details are introduced to ensure feasibility, efficiency, and performance. 
First, the generator produces $k$ new queries $\{\tilde{\xb}_{j}^{(i)}\}_{j=1}^k$ 
for a given input query $\xb^{(i)}$. To preserve the original label of the seed data, we use two distinct prompts $\bc_{y:y=\pm 1}$ for generating samples, based on whether the input is safe or unsafe: 
$\tilde{\xb}^{(i)} \sim p_{\bphi_{t-1}}(\tilde{\xb} | \xb^{(i)}, \bc_{y^{(i)}})$. 
We detail the prompts used for the generator in Appendix~\ref{app:exp}. 
The training data $\mathcal{S}^{(t)}$ at iteration $t$ is augmented exclusively with misclassified synthetic samples, defined as: $\mathcal{S}^{(t)} = \mathcal{S}^{(t-1)} \cup \tilde{\mathcal{S}}_{\text{mis}}$, where $\tilde{\mathcal{S}}_{\text{mis}} = \{\tilde{\xb}_j^{(i)}: \hat{y}_j^{(i)} \neq y^{(i)}\}$ and $\mathcal{S}^{(0)}=\mathcal{S}$.

To further enhance performance, we adopt a fine-grained multi-label classification setup similar to \citet{dubey2024llama}, where harmful inputs can have multiple labels (e.g., hate, violence), and safe content is labeled with all zeros. The classifier’s objective is modified to a multi-label classification loss using binary cross-entropy loss (equivalent to the negative log-likelihood minimization) for each of the 12 defined harmful classes (detailed in Appendix~\ref{app:data}): \vspace{-2mm}

\begin{align}
&L^{(t)}_{\mathcal{C}}(\btheta) = - \frac{1}{|\mathcal{S}^{(t)}|} \sum_{\left(\tilde{\xb},\{y_c\}\right) \in \mathcal{S}^{(t)}} \sum_{c=1}^{12} \nonumber\\
&\bigg[y_c \log p_{\btheta}(y_c | \tilde{\xb}) + (1 - y_{c}) \log (1 - p_{\btheta}(y_{c} | \tilde{\xb})) \bigg]. \label{eq:entropy_loss}
\end{align}

To maintain stability, we retrain the classifier from scratch at each iteration using the evolving dataset, similar to iterative approaches in mathematical reasoning~\citep{hosseini2024v}.

For the generator, the DPO training objective increases the likelihood of preferred data, which are samples that cause incorrect prediction of the classifier. Therefore, we consider the correctly classified ones as the dispreferred generation samples in preference learning. The correctly classified samples are defined as $\tilde{\mathcal{S}}_{\text{cor}} = \{\tilde{\xb}_j^{(i)} : \hat{y}_j^{(i)} = y^{(i)}\}$. 
The generator's loss is then given by:
\begin{align}
    &L^{(t)}_{\mathcal{G}}(\bphi, \bphi_{\text{ref}}) = \frac{1}{N}\sum_{\xb\in\mathcal{S}^{(t)}, \tilde{\xb}_w \in \tilde{\mathcal{S}}_{\text{mis}}, \tilde{\xb}_l \in \tilde{\mathcal{S}}_{\text{cor}}}\nonumber \\
    &\quad \bigg[\ell\bigg(\beta \log \frac{p_{\bphi}(\tilde{\xb}_w | \xb)}{p_{\bphi_{\text{ref}}}(\tilde{\xb}_w | \xb)} - \beta \log \frac{p_{\bphi}(\tilde{\xb}_l | \xb)}{p_{\bphi_{\text{ref}}}(\tilde{\xb}_l | \xb)}\bigg)\bigg],\label{eq:dpo_loss}
\end{align}
where $N < |\mathcal{S}^{(t)}|$ is the number of preference pairs that we were able to construct. We summarize the practical algorithm in Algorithm~\ref{alg:adversarial-training} and further technical details in Appendix~\ref{app:data_cur}. 

\begin{algorithm}[ht]
\caption{Two-Player Training}
\label{alg:adversarial-training}
\textbf{Require:} Initial generator $\mathcal{G}_{\bphi_0}$ and classifier $\mathcal{C}_{\btheta_0}$; maximum iteration $T$. \\
\textbf{Input:} Seed training dataset $\mathcal{S} = \{(\xb^{(i)}, y^{(i)})\}_{i=1}^N$. Prompt $\bc_{y=-1}$ and $\bc_{y=1}$.\\
\textbf{Output:} Final generator $\mathcal{G}_{\bphi_T}$ and classifier $\mathcal{C}_{\btheta_T}$.
\begin{algorithmic}[1]
\For{$t = 1, \dots, T$}
    \State \textbf{Sample Queries:}
    \For{$(\xb^{(i)}, y^{(i)}) \in \mathcal{S}$}
        \State Sample $\{\tilde{\xb}_j^{(i)}\}_{j=1}^k \sim p_{\bphi_{t-1}}(\tilde{\xb} | \bc_{y^{(i)}},\xb^{(i)})$ 
        \State Assign $\hat{y}_j^{(i)} = \mathcal{C}_{\btheta_{t-1}}(\tilde{\xb}_j^{(i)})$ 
        \State Partition into:
        \[
        \tilde{\mathcal{S}}_{\text{mis}}^{(i)} = \{\tilde{\xb}_j^{(i)} : \hat{y}_j^{(i)} \neq y^{(i)}\},\]
        \[\tilde{\mathcal{S}}_{\text{cor}}^{(i)} = \{\tilde{\xb}_j^{(i)} : \hat{y}_j^{(i)} = y^{(i)}\}
        \]
    \EndFor

    \State \textbf{Update Classifier according to \eqref{eq:entropy_loss}:} $\btheta_t \gets \argmin_{\btheta} L_{\mathcal{C}}(\btheta)$

    \State \textbf{Update Generator according to \eqref{eq:dpo_loss}:} $\bphi_t \gets \argmin_{\bphi} L_{\mathcal{G}}(\bphi, \bphi_{\text{ref}})$

\EndFor
\State \Return $\mathcal{C}_{\btheta_T}$
\end{algorithmic}
\end{algorithm}



\section{EXPERIMENTS}\label{sec:exp}
\textbf{Setup.}  
In our experiments, we use Qwen2.5-0.5B and Qwen2.5-1.5B~\citep{qwen2.5} as the base models for the classifier, since  the guardrail model is typically small-scale model and Qwen2.5-0.5B and Qwen2.5-1.5B models are among the most effective small-scale multilingual models available. In addition, we use dolphin-2.9.4-llama3.1-8b\footnote{https://huggingface.co/cognitivecomputations/dolphin-2.9.4-llama3.1-8b} as the base model for the generator, which is an uncensored multilingual model that meets our requirements for generating harmful queries in multiple languages. We follow the optimization process outlined in Section~\ref{sec:practical} and Algorithm~\ref{alg:adversarial-training} to train both models, applying full fine-tuning to the classifier and generator.  
For baselines, we compare against specialized guardrail models, including LlamaGuard3~\citep{inan2023llama} (1B) and ShieldGemma~\citep{zeng2024shieldgemma} (2B), which are SOTA models of similar scale to \ours. Additionally, we include larger-scale versions of LlamaGuard2 (8B) and LlamaGuard3 (8B) for a more comprehensive comparison. We detail the hyperparameters in Appendix~\ref{app:exp}. 
 

\textbf{Data.} To construct the seed dataset, we gather and combine training data from existing open-source data related to safety and toxicity, with detailed source information provided in Appendix~\ref{app:data}. We note that, instruction-following and QA data in sensitive domains (e.g., medical, legal, political) were also selected as benign examples containing potentially sensitive keywords. To prevent the classifier from relying on superficial keyword cues, we downsampled harmful examples dominated by specific terms. Harmful examples were further categorized into 12 groups, with an LLM assisting in labeling when category boundaries were ambiguous. Duplicate entries were removed to avoid overrepresentation, and the corpus was decontaminated to ensure no overlap with test data. 
The final linguistic composition of our gathered open-source dataset reveals a pronounced linguistic imbalance, where English data takes $81.4\%$ (1,679,516 instances), substantially predominating over French as $8.9\%$ (183,919), Spanish as $5.2\%$ (107,052), and German as $4.5\%$ (92,793). For generating the synthetic data, we set a temperature of 0.7 to encourage more diverse and creative generations and consider $k=8$.

\textbf{Evaluation.} We evaluate our method in four languages: English, French, German, and Spanish. For benchmarking guardrail models, we use six safety datasets: XSTest~\citep{rottger2023xstest}, ToxicChat~\citep{lin2023toxicchat}, OpenAI Moderation~\citep{markov2023holistic}, Beavertails~\citep{ji2024beavertails}, RTP-LX~\citep{de2024rtp}, and XSafety~\citep{wang2023all}. Among these, RTP-LX and XSafety are dedicated multilingual safety benchmarks, while the remaining four (XSTest, ToxicChat, OpenAI Moderation, and Beavertails) are commonly used English safety benchmarks. To enable multilingual evaluation, we translate these four datasets into languages that we considered.

\begin{table*}[!ht]
\vspace{-2mm}
    \centering
    \caption{Detailed F-1 scores on the classification benchmarks. The \textbf{bold} numbers indicate the best results among the methods evaluated and the \underline{underscored} numbers represent the second-best results. In the table, we abbreviate LlamaGuard as LG and ShieldGemma as SG.}
    \resizebox{\linewidth}{!}{%
\begin{tblr}{colspec = {cccccccccccccccc},
row{1-2, 8-9} = {bg=gray!25},
row{4, 6, 11, 13} = {bg=gray!10}
}
    \toprule 
    \SetCell[r=2]{c}{Model} & \SetCell[r=2]{c}{Size} $\downarrow$ & \SetCell[c=7]{c}{English} $\uparrow$ & & & & & & & \SetCell[c=7]{c}{German} $\uparrow$ & & & & & & \\
    \cmidrule[lr]{3-9} \cmidrule[lr]{10-16}   
    & & XSTest & OpenAI & ToxicC. & BeaverT. & RTP-LX & XSafety & \textbf{Avg} & XSTest & OpenAI & ToxicC. & BeaverT. & RTP-LX & XSafety & \textbf{Avg}\\
    \midrule
    LG3 & 1B & 43.4 & 36.8 & 22.3 & 51.6 & \underline{54.6} & \textbf{62.3} & 45.2 & 43.0 & 37.4 & 20.9 & 50.2 & \underline{55.4} & \textbf{61.4} & 44.7 \\
    SG & 2B & 69.4 & 44.8 & 36.4 & 51.6 & 26.0 & 30.6 & 43.1 & 59.6 & 38.7 & 27.5 & 51.6 & 19.5 & 24.1 & 36.8 \\
    LG2 & 8B & \textbf{88.8} & \underline{75.9} & 46.3 & \underline{72.3} & 39.5 & 35.2 & 59.7 & \underline{79.8} & \underline{74.4} & 40.5 & 68.5 & 38.7 & 30.6 & 55.4 \\
    LG3 & 8B & \underline{88.4} & \textbf{79.0} & \underline{54.0} & 70.1 & 48.5 & 40.5 & \underline{63.4} & \textbf{82.9} & \textbf{78.5} & \underline{48.0} & \underline{70.4} & 50.2 & 37.8 & \underline{61.3} \\
    \midrule
    Ours & \textbf{0.5B} & 82.3 & 70.8 & \textbf{70.1} & \textbf{86.1} & \textbf{91.7} & \underline{48.5} & \textbf{74.9} & 75.8 & 65.9 & \textbf{61.4} & \textbf{80.8} & \textbf{87.3} & \underline{60.4} & \textbf{71.9} \\
    \bottomrule
    \toprule 
    \SetCell[r=2]{c}{Model} & \SetCell[r=2]{c}{Size} $\downarrow$ & \SetCell[c=7]{c}{French} $\uparrow$ & & & & & & & \SetCell[c=7]{c}{Spanish} $\uparrow$ & & & & & & \\
    \cmidrule[lr]{3-9} \cmidrule[lr]{10-16}   
    & & XSTest & OpenAI & ToxicC. & BeaverT. & RTP-LX & XSafety & \textbf{Avg} & XSTest & OpenAI & ToxicC. & BeaverT. & RTP-LX & XSafety & \textbf{Avg}\\
    \midrule
    LG3 & 1B & 43.0 & 37.8 & 19.5 & 50.9 & \underline{54.9} & \textbf{61.3} & 44.6 & 46.9 & 37.9 & 20.4 & 50.3 & \underline{52.1} & \textbf{62.1} & 45.0 \\
    SG & 2B & 63.3 & 36.8 & 28.7 & 50.1 & 21.5 & 23.9 & 37.4 & 62.4 & 37.7 & 29.1 & 50.8 & 17.8 & 24.0 & 37.0 \\
    LG2 & 8B & \underline{81.6} & \underline{74.5} & 39.7 & 68.6 & 40.0 & 35.4 & 56.6 & \underline{84.0} & \underline{74.8} & 39.2 & 67.5 & 39.4 & 33.8 & 56.5 \\
    LG3 & 8B & \textbf{84.4} & \textbf{78.1} & \underline{50.1} & \underline{69.5} & 48.8 & 40.3 & 61.9 & \textbf{86.2} & \textbf{77.7} & \underline{48.4} & 69.5 & 48.4 & 39.0 & \underline{61.5} \\
    \midrule
    Ours & \textbf{0.5B} & 79.2 & 67.1 & \textbf{62.8} & \textbf{81.3} & \textbf{91.0} & \underline{54.7} & \textbf{72.7} & 81.4 & 66.8 & \textbf{64.9} & \textbf{81.4} & \textbf{88.0} & \underline{61.0} & \textbf{73.9}\\
    \bottomrule
    \end{tblr}%
    }
    \label{tab:main}
    \vspace{-4mm}
\end{table*}

\subsection{Main Results}
We present our main results in Figure~\ref{fig:main-res} and detail the performance on each dataset for each language in Table~\ref{tab:main}. \ours demonstrates significant advantages over existing guardrail models in both performance and efficiency. As shown in Figure~\ref{fig:main-res}, \ours achieves the highest average F1 score across English, French, Spanish, and German, outperforming all baselines, including the larger-scale LlamaGuard3 (8B) model, by over 10\%. Compared to models of similar scale, such as LlamaGuard3 (1B) and ShieldGemma (2B), \ours surpasses their performance by more than 30\% on average. Additionally, \ours exhibits the lowest inference cost (16.47 ms/input), achieving over a 2.5$\times$ speedup compared to LlamaGuard3 (8B) (58.88 ms/input) and ShieldGemma (2B) (57.83 ms/input). This highlights the efficiency of our approach, as it not only surpasses larger models in multilingual safety performance but also maintains significantly lower computational overhead, making it more practical for real-world deployment. In Figure~\ref{fig:perf_decline}, we present the average performance of each model across the three non-English languages relative to the English performance of our model \ours. Here, \ours achieves the lowest performance decline across all languages as compared to the English performance.

\begin{figure}[!ht]
    \centering
    \includegraphics[width=0.9\linewidth]{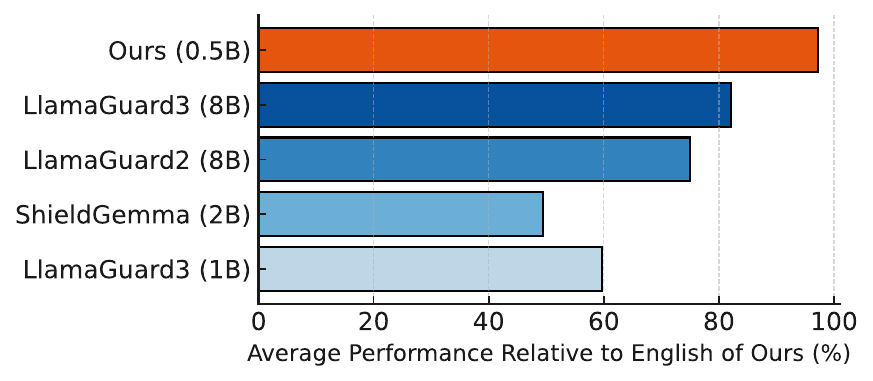}
    \caption{Relative performance decline (average F1 across 6 benchmarks and 3 languages) of various models compared to the En performance of \ours.}
    \label{fig:perf_decline}
\end{figure}

\begin{table}[!ht]
    \centering
    \caption{Average F1 scores across languages of different models trained with the dataset produced by our two-player scheme. Results generalize across base models and scales.}
    \label{tab:main-2}
    \resizebox{\linewidth}{!}{%
      \begin{tblr}{
          colspec={cccccc},
          row{1} = {bg=gray!25},
          row{3,5} = {bg=gray!10}
        }
        \toprule
        Model & Base & Size & En & Fr & Es & De\\
        \midrule
        LlamaGuard3 & Llama-3.2 & 1 B & 45.2 & 44.6 & 45.0 & 44.7 \\
        \ours & Llama-3.2 & 1B & \textbf{75.7} & \textbf{74.4} & \textbf{71.7} & \textbf{71.3} \\ 
        \midrule
        \ours        & Qwen-2.5  & 0.5 B & 74.9 & 71.9 & 72.7 & 73.9 \\
        \ours        & Qwen-2.5  & 1.5 B & \textbf{76.2} & \textbf{75.0} &
                       \textbf{73.7} & \textbf{74.0} \\
        \bottomrule
      \end{tblr}%
    }
\end{table}


\subsection{Weak-to-Strong Generalization}
Weak-to-strong generalization refers to the ability of a weaker model to generalize in supervising the training of stronger models. In Table~\ref{tab:main-2}, we leverage the training data generated by our two-player framework to train Llama-3.2 (1B), the base model for LlamaGuard3 (1B), and Qwen-2.5 (1.5B), a larger-scale model used to evaluate the weak-to-strong generalization capabilities of our method. 
We draw the following observations: (1) While the final fine-tuning results vary across base models, the data generated by our framework generalizes effectively across architectures, consistently outperforming baselines trained on the same base model by more than 20\%. (2) The two-player framework demonstrates weak-to-strong generalization, as data generated with the 0.5B classifier significantly improves the performance of the 1.5B classifier.

\section{ABLATION STUDY}
\subsection{Seed Data}
\textbf{Benefit of Incorporating Multilingual Data.} 
We evaluate three training configurations using only the seed dataset: training on English data alone, training on English and French data, and training on all four languages. Figure~\ref{fig:training_heatmap} presents the F1 scores on the OpenAI moderation test set for models trained under these conditions, all based on the \texttt{Qwen2.5-0.5B} model.  
Interestingly, training exclusively on English provides a relatively strong foundation for performance on French but is weaker on Spanish and German. Incorporating French data significantly improves performance on the French-translated OpenAI test set (from 51.3 to 65.2) while also enhancing performance on the Spanish- and German-translated test sets by 7.4 and 12.9 points, respectively. Additionally, English and French data appear to be mutually beneficial. 
The inclusion of Spanish and German data further improves performance on their respective test sets. However, as their addition reduces the proportion of English and French data, it leads to a slight performance decline overall.


\begin{figure}[!ht]
  \centering
  \begin{minipage}[t]{0.47\textwidth}
    \centering
    \includegraphics[width=0.8\linewidth]{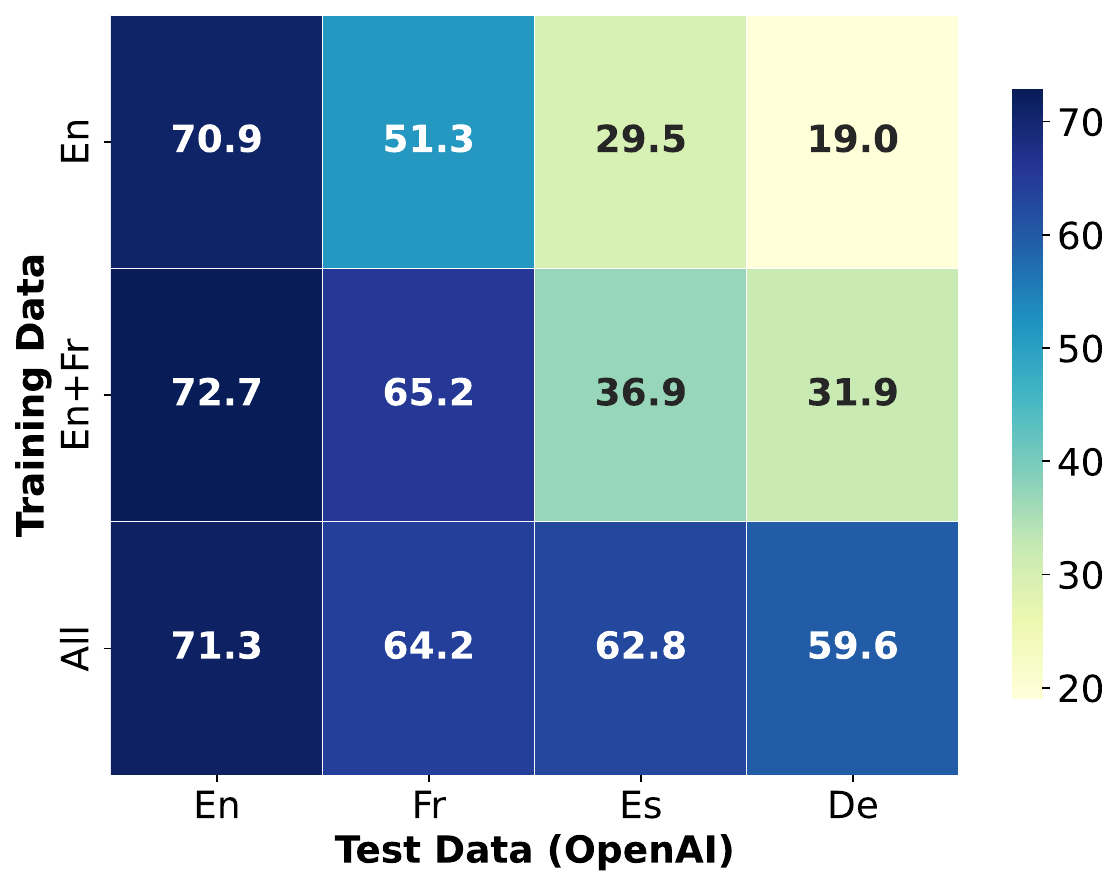}
    \caption{The F1 score on OpenAI benchmark of models trained with data containing different languages in our seed data. The inclusion of French in addition to English improves model performance on Spanish and German.}
    \label{fig:training_heatmap}
  \end{minipage}
  \hfill
  \begin{minipage}[t]{0.47\textwidth}
    \vspace{0pt}
    \centering
    \includegraphics[width=\linewidth]{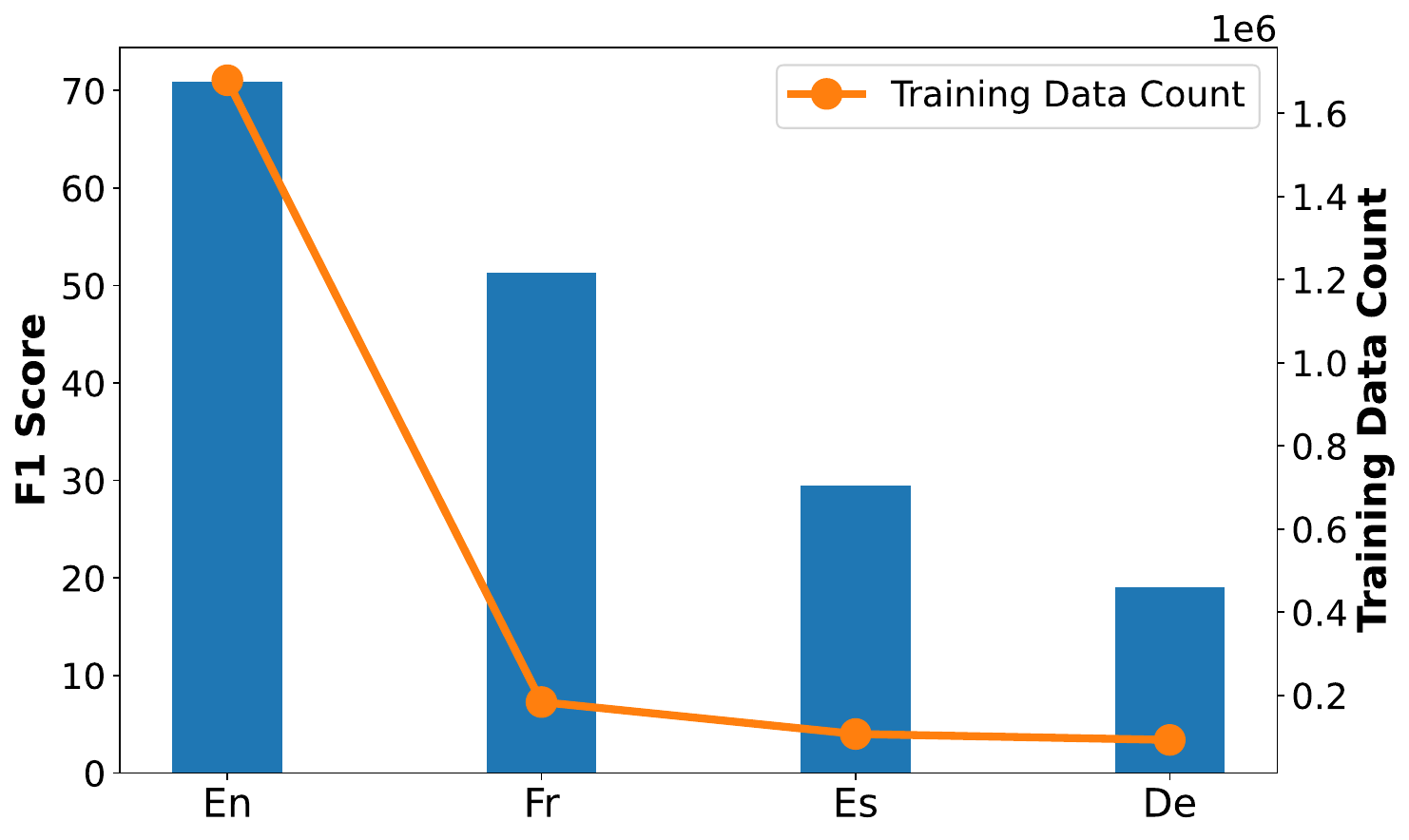}
    \caption{Performance  by languages of model trained on seed data. With larger data proportion in seed data, the model's average performance on English is markedly higher than other languages.}
    \label{fig:acc_prop}
  \end{minipage}
\end{figure}


\textbf{Performance Differences Due to Disproportionate Data.}  
Figure~\ref{fig:acc_prop} illustrates the relationship between training data volume per language and model performance (average F1 scores) across six benchmarks. The model is trained on the entire seed dataset, without synthetic data augmentation. The horizontal axis represents languages (English, French, Spanish, and German), while the left and right vertical axes indicate F1 scores and training data volume in the seed data, respectively.  
A clear trend emerges: languages with larger training datasets (e.g., English) achieve higher F1 scores, while those with less data (e.g., Spanish, German) perform worse. Although the performance gap varies across test sets, F1 scores consistently decline with reduced dataset size. This underscores the importance of synthetic data in mitigating performance disparities for low-resource languages.  
While the base LLM (Qwen-2.5 in our case) may have inherent limitations on low-resource languages, our method and the results of \ours demonstrate that incorporating synthetic multilingual data during post-training can significantly reduce this gap for the downstream task we consider.

\subsection{Synthetic Data}
\textbf{Iterative Improvement.} 
In Figure~\ref{fig:iterative}, we demonstrate the iterative improvement of the guardrail classifier in average F1 scores across English (En), French (Fr), Spanish (Es), and German (De) on the 6 benchmarks. Starting from iteration 0, which represents the baseline performance of training on seed data, substantial improvements are observed for all non-English languages after the first iteration. We particularly observe large gains in Spanish and German, highlighting the effectiveness of the iterative process in bridging performance gaps for lower-resource languages. By iteration 2, the performance for all languages converges, with Spanish and German achieving scores comparable to French, and all non-English languages narrowing the gap with English. In Figure~\ref{fig:data-dist}, we further show the data proportion across languages for iteration 0 (seed data) and synthetic data generated at iteration 1. At iteration 0, English dominates with 81\% of the data, while other languages (French, German, and Spanish) collectively account for less than 20\%. At iteration 1, the distribution for synthetic balances with the seed data, with English decreasing to 13\%, and significant increases in French (27\%), German (35\%), and Spanish (24\%). 

\begin{figure}[!ht]
    \centering
    \vspace{-7mm}
    \subfigure[]{
        \includegraphics[width=0.41\linewidth]{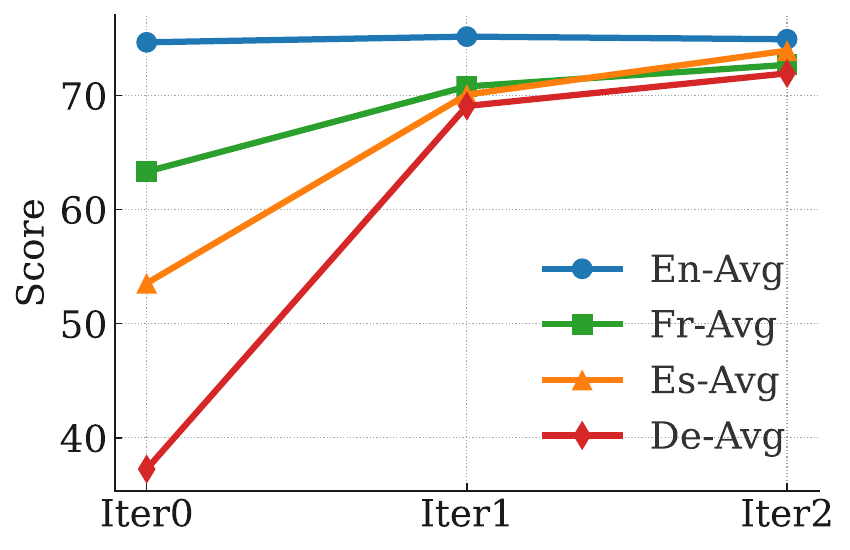}
        \label{fig:iterative}
    }
    \hfill
    \subfigure[]{
        \includegraphics[width=0.41\linewidth]{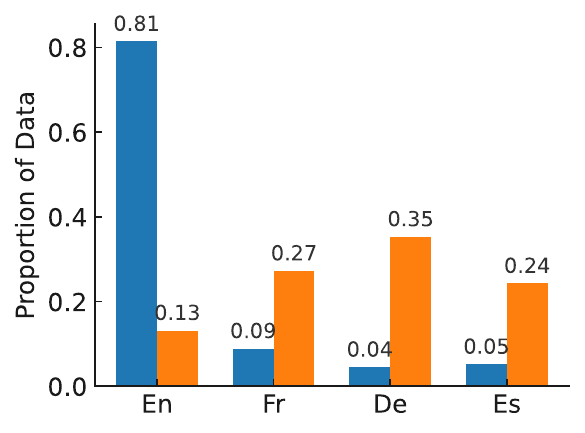}
        \label{fig:data-dist}
    }
    \caption{(a) Iterative performance improvements of \ours. (b) Shift in data distribution across languages over iterations.}
    \label{fig:combined}
\end{figure}

\section{CONCLUSION}\label{sec:conclude}

Our work addresses the data scarcity challenge in multilingual LLM safety through a novel self-improving framework that integrates synthetic data generation with guardrail model training. Specifically, we propose a two-player reinforcement learning approach formalized as a minimax game, providing theoretical guarantees of convergence. Empirical evaluations across six languages demonstrate that our model outperforms similarly-sized baselines by over 20\%, and larger models by 10\%, while maintaining a compact 0.5B parameter size and achieving a 3$\times$ inference speedup compared to existing guardrails.

\textbf{Limitations.} We note that synthetic data generation methods inherently depend on the quality of the underlying LLM, with stronger models naturally yielding superior outcomes. 
Modern LLMs, such as Qwen-2.5, already support over 29 languages, highlighting the potential for generating robust multilingual post-training datasets. While our current implementation focuses specifically on English, French, German, and Spanish to illustrate the efficacy of our two-player data synthesis framework, the approach inherently retains the full multilingual capacity of Qwen-2.5, thereby supporting extensive future expansions to additional languages.

\section*{Acknowledgements}
This work was partially supported by NIH U54OD036472, U54DK097771, U54HG012517, NSF 2312501, 2106859, Amazon, NEC, Optum.

\bibliography{main}
\bibliographystyle{plainnat}

\section*{Checklist}

\begin{enumerate}

  \item For all models and algorithms presented, check if you include:
  \begin{enumerate}
    \item A clear description of the mathematical setting, assumptions, algorithm, and/or model. [Yes]
    \item An analysis of the properties and complexity (time, space, sample size) of any algorithm. [Not Applicable]
    \item (Optional) Anonymized source code, with specification of all dependencies, including external libraries. [Yes]
  \end{enumerate}

  \item For any theoretical claim, check if you include:
  \begin{enumerate}
    \item Statements of the full set of assumptions of all theoretical results. [Yes]
    \item Complete proofs of all theoretical results. [Yes]
    \item Clear explanations of any assumptions. [Yes]     
  \end{enumerate}

  \item For all figures and tables that present empirical results, check if you include:
  \begin{enumerate}
    \item The code, data, and instructions needed to reproduce the main experimental results (either in the supplemental material or as a URL). [Yes]
    \item All the training details (e.g., data splits, hyperparameters, how they were chosen). [Yes]
    \item A clear definition of the specific measure or statistics and error bars (e.g., with respect to the random seed after running experiments multiple times). [Not Applicable]
    \item A description of the computing infrastructure used. (e.g., type of GPUs, internal cluster, or cloud provider). [Yes]
  \end{enumerate}

  \item If you are using existing assets (e.g., code, data, models) or curating/releasing new assets, check if you include:
  \begin{enumerate}
    \item Citations of the creator If your work uses existing assets. [Yes]
    \item The license information of the assets, if applicable. [Not Applicable]
    \item New assets either in the supplemental material or as a URL, if applicable. [Not Applicable]
    \item Information about consent from data providers/curators. [Yes]
    \item Discussion of sensible content if applicable, e.g., personally identifiable information or offensive content. [Yes]
  \end{enumerate}

  \item If you used crowdsourcing or conducted research with human subjects, check if you include:
  \begin{enumerate}
    \item The full text of instructions given to participants and screenshots. [Not Applicable]
    \item Descriptions of potential participant risks, with links to Institutional Review Board (IRB) approvals if applicable. [Not Applicable]
    \item The estimated hourly wage paid to participants and the total amount spent on participant compensation. [Not Applicable]
  \end{enumerate}

\end{enumerate}

\clearpage
\appendix
\thispagestyle{empty}

\onecolumn
\aistatstitle{Appendix}
\section{Theoretical Analysis}\label{app:proof}
In this section, we provide a detailed theoretical analysis about our two-player minimax game framework.
\subsection{Minimizer of Loss}
First of all, we derive the solution of the optimization objectives defined in Equations \eqref{eq:update_classifier} and \eqref{eq:population_dpo}.
\subsubsection{Generator}
Recall that the corresponding RL optimization objective of DPO objective \eqref{eq:population_dpo} is:
\begin{align}
    \EE_{(\xb,y)\sim\rho(\xb,y)}  \big[ \EE_{\tilde{\xb}\sim p_{\bphi}(\tilde{\xb} | \xb, y)}[r_t(\xb, \tilde{\xb})] - \beta D_{\text{KL}}(p_{\bphi} | p_{\text{ref}}) \big],\label{eq:ppo}
\end{align}
where $\rho(\xb,y)$ is the data distribution and $r_t((\xb,y), \tilde{\xb}) = - \log p_{\btheta_t}(y|\tilde{\xb})$ is the reward function defined in \eqref{eq:reward}. We will show that the DPO objective \eqref{eq:population_dpo} and the KL-regularized reward maximization objective~\ref{eq:ppo} shares the same minimizer. \citet{azar2024general} provided the following connection between the RL and DPO objectives.
\begin{proposition}[Proposition 4 in \citet{azar2024general}]\label{prop:same_minimizer}
Let the DPO training objective be 
\begin{align*}
    L_1(\bphi, \bphi_{\text{ref}}) = \EE_{\xb\sim\rho}\EE_{\yb_w, \yb_l \sim \mu(\cdot|\xb)} \bigg[ \PP(\yb_w \succ \yb_l | \xb) 
 \ell\bigg(\beta \log \frac{p_{\bphi}(\yb_w | \xb)}{p_{\bphi_{\text{ref}}}(\yb_w | \xb)} - \beta \log \frac{p_{\bphi}(\yb_l | \xb)}{p_{\bphi_{\text{ref}}}(\yb_l | \xb)}\bigg)\bigg],
\end{align*}
and the RLHF training objective be
\begin{align*}
    & L_2(\bphi,\bphi_{\text{ref}}) = \EE_{\xb\sim\rho(\xb)}{\EE_{\yb \sim p_{\bphi}(\cdot|\xb)}[r(\yb, \xb)]} - {\beta D_{\text{KL}}(p_{\bphi} | p_{\text{ref}})}.
\end{align*}
    Consider a preference model \( p^* \) such that there exists a minimizer to the Bradley-Terry loss
\[
\arg\min_r - \mathbb{E}_{\xb \sim \rho} \mathbb{E}_{\yb_w, \yb_l \sim \mu(\cdot | \xb)} \left[ p^*(\yb_w \succ \yb_l | \xb) \log \sigma(r(\xb, \yb_w) - r(\xb, \yb_l)) \right].
\]
Then, the optimal policy for the DPO objective and for the RLHF objective with the reward model given as the minimizer to the Bradley-Terry loss above are identical, regardless of whether or not \( p^* \) corresponds to a Bradley-Terry preference model.
\end{proposition}

Therefore, we only need to show that the reward function is the minimizer of the Bradley-Terry loss.
\begin{lemma}\label{lm:minimize_r}
Let $\sigma$ be the sigmoid function and $p^* (\tilde{\xb}_{w} \succ \tilde{\xb}_{l} |\xb,y ) = \sigma\bigl(r^*((\xb,y),\tilde{\xb}_w)-r^*((\xb,y),\tilde{\xb}_l)\bigr) $. Then, we have
\begin{align*}
    \argmin_{r} \mathbb{E}_{\substack{(\xb,y) \sim \rho(\xb,y) \\ \tilde{\xb}_w, \tilde{\xb}_l \sim p_{\phi_n}(\cdot | \xb, y)}} \bigg[ -p^* (\tilde{\xb}_{w} \succ \tilde{\xb}_{l} |\xb,y ) \log\sigma\big(r((\xb,y) \tilde{\xb}_w) - r((\xb,y), \tilde{\xb}_l)\big) \bigg] = r^*((\xb,y), \tilde{\xb}) + c(\xb,y).
\end{align*}
\end{lemma}
\begin{proof}[Proof of Lemma~\ref{lm:minimize_r}]
The objective can be viewed as a cross-entropy between the distribution \(p^*(\tilde{\xb}_w \succ \tilde{\xb}_l \mid \xb, y)\) and \(\sigma\bigl(r((\xb,y),\tilde{\xb}_w)-r((\xb,y),\tilde{\xb}_l)\bigr)\). In particular, the objective depends only on the difference \(r((\xb,y), \tilde{\xb}_w) - r((\xb,y), \tilde{\xb}_l)\). Hence the value of the objective doesn't change if we replace \(r\) by $\tilde{r}((\xb,y), \tilde{\xb})\;=\;r((\xb,y), \tilde{\xb}) + c(\xb,y)$. The function \(p^* (\tilde{\xb}_{w} \succ \tilde{\xb}_{l} |\xb,y )\) is given by the sigmoid 
\[
p^* (\tilde{\xb}_{w} \succ \tilde{\xb}_{l} |\xb,y ) \;=\;\sigma\bigl(r^*((\xb,y), \tilde{\xb}_w) - r^*((\xb,y), \tilde{\xb}_l))\bigr).
\]
Minimizing the cross-entropy is achieved exactly when
\[
\sigma\bigl(r((\xb,y) \tilde{\xb}_w) - r((\xb,y) \tilde{\xb}_l)\bigr)
\;=\;
\sigma\bigl(r^*((\xb,y), \tilde{\xb}_w) - r^*((\xb,y), \tilde{\xb}_l)\bigr)
\]
for all \(\xb, \tilde{\xb}_w, \tilde{\xb}_l, y\). Since the sigmoid is strictly increasing, we have
\[
r((\xb,y), \tilde{\xb}_w) - r((\xb,y), \tilde{\xb}_l) \;=\; r^*((\xb,y), \tilde{\xb}_w) - r^*((\xb,y), \tilde{\xb}_l).
\]
The solution is 
\[
r((\xb,y), \tilde{\xb}) \;=\; r^*((\xb,y), \tilde{\xb}) + c(\xb,y).
\]
\end{proof}

Then, by Proposition~\ref{prop:same_minimizer} and Lemma~\ref{lm:minimize_r}, the DPO objective \eqref{eq:population_dpo} shares the same minimizer with its corresponding RL training objective~\eqref{eq:ppo}. In addition, according to \citet{rafailov2023direct}, the minimizer is
\begin{align*}
    p_{\phi_{n+1}}(\tilde{\xb}|\xb,y) = \frac{1}{Z(\xb,y)}{p_{\text{ref}}(\tilde{\xb}|\xb,y) \exp(\beta^{-1} [-\log p_{\theta_{n}}(y|\tilde{\xb})] )} \propto  p_{\text{ref}}(\tilde{\xb}|\xb,y) \exp(\beta^{-1} [-\log p_{\theta_{n}}(y|\tilde{\xb})] ), 
\end{align*}
where $Z(\xb,y) = \mathbb{E}_{\tilde{\xb}\sim p_{\text{ref}}(\tilde{\xb} |\xb,y)} \exp(\beta^{-1} [-\log p_{\theta_{n}}(y|\tilde{\xb})] )$ is the normalization term.
\subsubsection{Classifier}
Next, we will derive the solution to the objective~\eqref{eq:update_classifier}. We first prove a tool lemma.
\begin{lemma}\label{lm:mle}
    Let $p(y,\tilde{\xb})$ be a joint distribution over $(y,\tilde{\xb})$. Then
    \begin{align*}
        \max_q \EE_{(y,\tilde{\xb}) \sim p(y,\tilde{\xb})} \big[ \log q(y | \tilde{\xb}) \big] = - H[p(y | \tilde{\xb})],
    \end{align*}
    and the maximizer is $q^*(y | \tilde{\xb}) = p(y | \tilde{\xb}) $. Here, H is the entropy.
\end{lemma}
\begin{proof}[Proof of Lemma~\ref{lm:mle}]
\begin{align*}
    \begin{aligned}
        L(q) & = \EE_{(y,\tilde{\xb}) \sim p(y,\tilde{\xb})} \big[ \log q(y | \tilde{\xb}) \big] \\
        & = \EE_{p(\tilde{\xb})}  \big[ \EE_{(y|\tilde{\xb}) \sim p(y,\tilde{\xb})} \big[ \log q(y | \tilde{\xb}) \big]\big] \\
        & = \EE_{p(\tilde{\xb})}  \big[ \EE_{(y|\tilde{\xb}) \sim p(y,\tilde{\xb})} \big[ \log p(y | \tilde{\xb}) \big] - D_{\text{KL}} (p(y | \tilde{\xb}) || q(y | \tilde{\xb}))\big] \\
        & \le \EE_{p(\tilde{\xb})}  \big[ \EE_{(y|\tilde{\xb}) \sim p(y,\tilde{\xb})} \big[ \log p(y | \tilde{\xb}) \big] \big],
    \end{aligned}
\end{align*}
    and the last equity holds if and only if $p(y | \tilde{\xb}) =  q(y | \tilde{\xb})$.
\end{proof}
Then, we can calculate the minimizer of \eqref{eq:update_classifier}.
\begin{lemma}\label{lm:nll_minimizer}
    ${\int \rho(\xb, y)  p_{\phi_n}(\tilde{\xb} | \xb, y) d\xb} \Big{/}{\int \rho(\xb,y)  p_{\phi_n}(\tilde{\xb} | \xb, y) d\xb dy}$ is the minimizer to the following optimization problem:
    \begin{align*}
        \argmin_q  \mathbb{E}_{\substack{(\xb,y) \sim \rho(\xb,y) \\ \tilde{\xb} \sim p_{\phi_n}(\tilde{\xb}|\xb, y)}} \big[- \log q (y|\tilde{\xb}) \big].
    \end{align*}
\end{lemma}
\begin{proof}[Proof of Lemma~\ref{lm:nll_minimizer}]
    The joint distribution of $(y,\tilde{\xb})$ is 
    \begin{align*}
        p(y,\tilde{\xb}) = \int \rho(\xb, y)  p_{\phi_n}(\tilde{\xb} | \xb, y) d\xb,
    \end{align*}
    and the marginal distribution of $\tilde \xb$ is
    \begin{align*}
        p(\tilde{\xb}) = \int \rho(\xb,y)  p_{\phi_n}(\tilde{\xb} | \xb, y) d\xb dy.
    \end{align*}
    We can restate the optimization problem as 
    \begin{align*}
        \argmax_{q} \EE_{(y,\tilde{\xb}) \sim p(y,\tilde{\xb})}  \big[\log q (y|\tilde{\xb}) \big].
    \end{align*} 
    By Lemma~\ref{lm:mle}, the solution is 
    \begin{align*}
        \begin{aligned}
            q (y|\tilde{\xb}) 
            = p(y | \tilde{\xb})
            = \frac{p(y , \tilde{\xb})}{p(\tilde{\xb})}
            = \frac{\int \rho(\xb, y)  p_{\phi_n}(\tilde{\xb} | \xb, y) d\xb}{\int \rho(\xb,y)  p_{\phi_n}(\tilde{\xb} | \xb, y) d\xb dy}.
        \end{aligned}
    \end{align*}
\end{proof}
Therefore, for the classifier, by Lemma~\ref{lm:nll_minimizer}, we have
\begin{align*}
    p_{\theta_{n+1}}(y | \tilde{\xb}) = \argmin_{q}   \mathbb{E}_{\substack{(\xb,y) \sim \rho(\xb,y) \\ \tilde{\xb} \sim p_{\phi_n}(\tilde{\xb} | \xb,y)}} [-\log q (y|\tilde{\xb})] = \frac{\int \rho(\xb, y)  p_{\phi_n}(\tilde{\xb} | \xb, y) d\xb}{\int \rho(\xb,y)  p_{\phi_n}(\tilde{\xb} | \xb, y) d\xb dy}.
\end{align*}
In a two player game perspective, $p_{\theta_{n+1}}$ can be viewed as the best response to $p_{\phi_{n}}$, and $p_{\phi_{n+1}}$ can be viewed as the best response to $p_{\theta_{n}}$.  For simplicity, we denote that $p_{\theta_{n+1}}= T_{\theta}(p_{\phi_n})$ and $p_{\phi_{n+1}}= T_{\phi}(p_{\theta_n})$.

\subsection{Nash Equilibrium}
\subsubsection{Existence of Nash Equilibrium}\label{app:existence}
In our two-player game framework, we indeed optimize the following minimax two player game:
\begin{align}
    & \min_\theta \max_\phi F(p_{\phi}, p_{\theta}) \nonumber \\
   F(p_{\phi}, p_{\theta}) =  \EE_{(\xb,y)\sim\rho(\xb,y)}  \big[ \EE_{\tilde{\xb}\sim p_{\bphi}(\tilde{\xb} | \xb, y)}&[- \log p_{\btheta_t}(y|\tilde{\xb})] - \beta D_{\text{KL}}(p_{\bphi}(\cdot|\xb,y) | p_{\text{ref}}(\cdot|\xb,y)) \big]. \label{eq:full_target}
\end{align}
We further enforce the following regularity conditions:
\begin{itemize}[leftmargin=*, topsep=0pt]
\item Both $\mathcal{X}$ and $\tilde{\mathcal{X}}$ are finite discrete sets of tokens, with $|\mathcal{X}| = X < \infty$ and $|\tilde{\mathcal{X}}| = \tilde{X} < \infty$. 
\item We constrain $p_\theta$ within a half-space of the Euclidean space, ensuring $p_\theta(y|\xb) \ge \gamma > 0$. 
\item The normalization term of the generator distribution is strictly positive:
    \[\sum_{\tilde\xb\in\tilde{\mathcal{X}}} p_{\text{ref}}(\tilde\xb|\xb,y) \exp\big(\beta^{-1} [-\log p_{\theta}(y|\tilde{\xb})] \big) \ge \delta > 0.\]
\item The distribution $p_\phi$ is non-degenerate, i.e., $ \sum_{y=\pm 1}\sum_{\xb \in \mathcal{X}} \rho(\xb, y) p_{\phi}(\tilde{\xb}|\xb,y) \ge \alpha >0$.
\end{itemize}

\begin{theorem}[Von Neumann's Minimax Theorem]
Let \( X \subseteq \mathbb{R}^n \) and \( Y \subseteq \mathbb{R}^m \) be {compact convex} sets. If \( f: X \times Y \to \mathbb{R} \) is a continuous function that is concave-convex, i.e.
\[
f(\cdot, y): X \to \mathbb{R} \text{ is } {\text{concave}} \text{ for every fixed } y \in Y, \text{ and}
\]
\[
f(x, \cdot): Y \to \mathbb{R} \text{ is } {\text{convex}} \text{ for every fixed } x \in X.
\]
Then we have that
\[
\max_{x \in X} \min_{y \in Y} f(x, y) = \min_{y \in Y} \max_{x \in X} f(x, y).
\]
\end{theorem}

Our optimization target $F(p_{\phi}, p_{\theta})$ in \eqref{eq:full_target} is concave on $p_{\phi}$ since the first term is linear in $p_{\phi}$ and $D_{\text{KL}}(p_{\phi} || p_{\text{ref}})$ is convex in $p_{\phi}$. In addition, $F(p_{\phi}, p_{\theta})$ is convex in $p_\theta$ since $\log$ is a concave function. By Von Neumann's Minimax Theorem, we have 
\begin{align*}
    \min_{p_{\btheta}} \max_{p_{\bphi}} F(p_{\phi}, p_{\theta}) = \max_{p_{\bphi}} \min_{p_{\btheta}}  F(p_{\phi}, p_{\theta}).
\end{align*}
We denote this value by $v$. In addition, we let 
\begin{align*}
    p^*_{\btheta} \in \argmin \max_{p_{\bphi}} F(p_{\bphi}, p_{\theta}), \\
    p^*_{\bphi} \in \argmax \min_{p_{\btheta}} F(p_{\bphi}, p_{\btheta}).
\end{align*}
That is,
\begin{align*}
    \max_{p_{\bphi}} F(p_{\bphi}, p^*_{\btheta}) = v, \\
    \min_{p_{\btheta}} F(p^*_{\bphi}, p_{\btheta}) = v.
\end{align*}

Then, we have
\begin{align*}
    \forall p_{\bphi} \quad F(p_{\bphi}, p^*_{\btheta}) \le \max_{p_{\bphi}} F(p_{\bphi}, p^*_{\btheta}) = v \\
    \forall p_{\btheta} \quad F(p^*_{\bphi}, p_{\btheta}) \ge \min_{p_{\btheta}} F(p^*_{\bphi}, p_{\btheta}) = v
\end{align*}

In addition, we would have $F(p^*_{\bphi}, p^*_{\btheta})$ = $v$ since $v \le F(p^*_{\bphi}, p^*_{\btheta}) \le v$. Therefore, we have for $p^*_{\bphi}, p^*_{\btheta}$ that
\begin{align*}
    \forall p_{\bphi},p_{\btheta} \quad  F(p_{\bphi}, p^*_{\btheta}) \le F(p^*_{\bphi}, p^*_{\btheta}) \le F(p^*_{\bphi}, p_{\btheta}),
\end{align*}
which satisfies the definition of Nash equilibrium.

\subsubsection{Convergence to Nash Equilibrium}\label{app:convergence}
In this section, we first show that both $T_\theta$ and $T_\phi$ are Lipschitz, and then we prove that our algorithm converges to the fixed point. 

\textbf{Lipschitz Mapping $T_\phi$.} Recall that
\begin{align*}
    T_\phi(p_\theta)(\tilde{\xb}|\xb,y) = \frac{p_{\text{ref}}(\tilde{\xb}|\xb,y) \exp(\beta^{-1} [-\log p_{\theta_{n}}(y|\tilde{\xb})] )}{\sum_{\tilde{\xb}} p_{\text{ref}}(\tilde{\xb} |\xb,y) \exp(\beta^{-1} [-\log p_{\theta_{n}}(y|\tilde{\xb})] )}.
\end{align*}
Let $g_\theta(\tilde\xb,y) = \exp\big(\beta^{-1} [-\log p_{\theta}(y|\tilde{\xb})] \big)$, by the regularity conditions, we have
\begin{align*}
    \bigg|\frac{\partial g_\theta (\tilde\xb,y)}{\partial p_\theta(y | \tilde{\xb})}\bigg|  = \Big|\beta^{-1} \big( p_\theta(y|\xb)\big)^{-1}\exp\big(\beta^{-1} [-\log p_{\theta}( y|\tilde{\xb})] \big)\Big| 
    \le \beta^{-1} \gamma^{-1-\beta^{-1}}.
\end{align*}
This leads to that
\begin{align*}
    |g_\theta(\tilde\xb,y) - g_{\theta'}(\tilde\xb,y)| \le \beta^{-1} \gamma^{-1-\beta^{-1}} |p_\theta(y|\tilde\xb) - p_{\theta'}(y|\tilde\xb)|.
\end{align*}
We rewrite 
\begin{align*}
    T_\phi(p_\theta)(\tilde\xb|\xb,y) = \frac{N_\theta(\tilde\xb,\xb,y)}{D_\theta(\xb,y)},\text{ where }N_\theta(\tilde\xb,\xb,y)=p_{\text{ref}}(\tilde{\xb}|\xb,y)g_\theta(\tilde\xb,y),\quad D_\theta(\xb, y) = \sum_{\tilde\xb\in\tilde{\mathcal{X}}} N_\theta(\tilde\xb,\xb,y).
\end{align*}

Then, we have
\begin{align*}
    & \sum_{\tilde\xb\in\tilde{\mathcal{X}}}|T_\phi(p_\theta) - T_\phi(p_{\theta'})|(\tilde\xb|\xb,y) \\
    & = \sum_{\tilde\xb\in\tilde{\mathcal{X}}} \bigg| \frac{N_\theta(\tilde\xb,\xb,y)}{D_\theta(\xb,y)} - \frac{N_{\theta'}(\tilde\xb,\xb,y)}{D_{\theta'}(\xb,y)} \bigg| \\
    & = \sum_{\tilde\xb\in\tilde{\mathcal{X}}} \bigg| \frac{N_\theta(\tilde\xb,\xb,y)}{D_\theta(\xb,y)} - \frac{N_{\theta'}(\tilde\xb,\xb,y)}{D_\theta(\xb,y)} + \frac{N_{\theta'}(\tilde\xb,\xb,y)}{D_\theta(\xb,y)}- \frac{N_{\theta'}(\tilde\xb,\xb,y)}{D_{\theta'}(\xb,y)} \bigg| \\
    & \le \sum_{\tilde\xb\in\tilde{\mathcal{X}}} \frac{|N_\theta(\tilde\xb,\xb,y) - N_{\theta'}(\tilde\xb,\xb,y) |}{D_\theta(\xb,y)} + \sum_{\tilde\xb\in\tilde{\mathcal{X}}} |N_{\theta'}(\tilde\xb,\xb,y) | \bigg|\frac{1}{D_{\theta}(\xb,y)} - \frac{1}{D_{\theta'}(\xb,y)}\bigg| \\
    & = \frac{1}{D_\theta(\xb,y)} \sum_{\tilde\xb\in\tilde{\mathcal{X}}} {|N_\theta(\tilde\xb,\xb,y) - N_{\theta'}(\tilde\xb,\xb,y) |} + \frac{D_{\theta'}(\xb,y)}{D_{\theta}(\xb,y) D_{\theta'}(\xb,y)}\bigg|  D_{\theta'}(\xb,y) - D_{\theta}(\xb,y)\bigg| \\
    & = \frac{1}{D_\theta(\xb,y)} \Big( \sum_{\tilde\xb\in\tilde{\mathcal{X}}} {|N_\theta(\tilde\xb,\xb,y) - N_{\theta'}(\tilde\xb,\xb,y) |} +  \Big|  D_{\theta'}(\xb,y) - D_{\theta}(\xb,y)\Big| \Big).
\end{align*}
And we have
\begin{align*}
|N_\theta(\tilde\xb,\xb,y) - N_{\theta'}(\tilde\xb,\xb,y)| \le     p_{\text{ref}}(\tilde{\xb}|\xb,y)|g_\theta(\tilde\xb,y) - g_{\theta'}(\tilde\xb,y)| \le |g_\theta(\tilde\xb,y) - g_{\theta'}(\tilde\xb,y)|,\\
|D_{\theta'}(\xb,y) -  D_\theta(\xb,y) | \le \sum_{\tilde\xb\in\tilde{\mathcal{X}}} p_{\text{ref}} (\tilde{\xb}|\xb,y)|g_\theta(\tilde\xb,y) - g_{\theta'}(\tilde\xb,y)|  \le \sum_{\tilde\xb\in\tilde{\mathcal{X}}} |g_\theta(\tilde\xb,y) - g_{\theta'}(\tilde\xb,y)|.
\end{align*}
In addition, by the regularity conditions, we have that
\begin{align*}
    & \sum_{\xb\in\mathcal{X}} \sum_{y=\pm 1} \sum_{\tilde\xb\in\tilde{\mathcal{X}}}|T_\phi(p_\theta) - T_\phi(p_{\theta'})|(\tilde\xb|\xb, y) \\ 
    &\le \sum_{\xb\in\mathcal{X}}  \sum_{y=\pm 1} \frac{2}{D_\theta(\xb,y)} \sum_{\tilde\xb\in\tilde{\mathcal{X}}} |g_\theta(\tilde\xb,y) - g_{\theta'}(\tilde\xb,y)|\\
    &\le \sum_{\xb\in\mathcal{X}}  \sum_{y=\pm 1} \frac{2}{\delta} \sum_{\tilde\xb\in\tilde{\mathcal{X}}} \beta^{-1} \gamma^{-1-\beta^{-1}} |p_\theta(y|\tilde\xb) - p_{\theta'}(y|\tilde\xb)| \\
    & = 2 \delta^{-1} \beta^{-1} \gamma^{-1-\beta^{-1}} |\mathcal{X}|\sum_{\tilde\xb\in\tilde{\mathcal{X}}} \sum_{y \in 
    \mathcal{Y}}|p_\theta(y|\tilde\xb) - p_{\theta'}(y|\tilde\xb)|.
\end{align*}
This means that 
\begin{align*}
    \|T_\phi(p_\theta) - T_\phi(p_{\theta'}) \|_1 \le 2 \delta^{-1} \beta^{-1} \gamma^{-1-\beta^{-1}} |\mathcal{X}| \| p_\theta - p_{\theta'} \|_1.
\end{align*}

\textbf{Lipschitz Mapping $T_{\btheta}$.} Recall that 
\begin{align*}
    T_\theta(p_\phi)(y|\tilde{\xb}) = \frac{\sum_{\xb \in \mathcal{X}} \rho(\xb, y)  p_{\phi}(\tilde{\xb}|\xb,y)}{ \sum_{\xb \in \mathcal{X}} \sum_{y=\pm1}\rho(\xb, y)  p_{\phi}(\tilde{\xb}|\xb,y)}.
\end{align*}
Denote that
\begin{align*}
    T_\theta(p_\phi)(y|\tilde{\xb}) = \frac{N_\phi(y,\tilde{\xb})}{D_\phi (\tilde{\xb})},\text{ where } N_\phi(y,\tilde{\xb}) = \sum_{\xb \in \mathcal{X}} \rho(\xb, y) p_{\phi}(\tilde{\xb}|\xb,y),\quad D_\phi (\tilde{\xb}) = \sum_{y = \pm 1} N_\phi(y,\tilde{\xb}).
\end{align*}
Then,
\begin{align*}
    \sum_{y \in \mathcal{Y}}| T_\theta(p_\phi) - T_\theta(p_{\phi'}) |(y|\tilde{\xb}) & = \sum_{y =\pm 1} \bigg| \frac{N_\phi(y,\tilde{\xb})}{D_\phi (\tilde{\xb})} -  \frac{N_{\phi'}(y,\tilde{\xb})}{D_{\phi'} (\tilde{\xb})}  \bigg| \\
    & \le \frac{1}{D_\phi (\tilde{\xb})}\bigg( \sum_{y=\pm1}\Big|N_\phi(y,\tilde{\xb}) -  N_{\phi'}(y,\tilde{\xb})\Big| + \Big|D_\phi (\tilde{\xb}) -  D_{\phi'} (\tilde{\xb}) \Big| \bigg).
\end{align*}
In addition,
\begin{align*}
    \Big|N_\phi(y,\tilde{\xb}) -  N_{\phi'}(y,\tilde{\xb})\Big| &\le \sum_{\xb \in \mathcal{X}} \rho(\xb, y) | p_{\phi}(\tilde{\xb}|\xb, y) - p_{\phi'}(\tilde{\xb}|\xb, y)|  \\
    \Big|D_\phi (\tilde{\xb}) -  D_{\phi'} (\tilde{\xb}) \Big| &\le \sum_{y=\pm1}\sum_{\xb \in \mathcal{X}} \rho(\xb, y) | p_{\phi}(\tilde{\xb}|\xb,y) - p_{\phi'}(\tilde{\xb}|\xb,y)|.
\end{align*}
Therefore,
\begin{align*}
    \sum_{\tilde\xb\in\tilde{\mathcal{X}}}\sum_{y = \pm 1}| T_\theta(p_\phi) - T_\theta(p_{\phi'}) |(y|\tilde{\xb}) \le \sum_{\tilde\xb\in\tilde{\mathcal{X}}} \frac{2}{D_\phi (\tilde{\xb})} \sum_{y=\pm1}\sum_{\xb \in \mathcal{X}} \rho(\xb, y) | p_{\phi}(\tilde{\xb}|\xb,y) - p_{\phi'}(\tilde{\xb}|\xb,y)|.  
\end{align*}
Let the marginal of $\rho(x)$ be a uniform distribution on the space $\mathcal{X}$. Then we have $\rho(\xb,y) \le 1 / |\mathcal{X}|$. By the regularity conditions, we have
\begin{align*}
    \| T_\theta(p_\phi) - T_\theta(p_{\phi'}) \|_1 \le 2 \alpha^{-1} |\mathcal{X}|^{-1} \| p_{\phi} - p_{\phi'} \|_1 .
\end{align*}

\textbf{Proof of Convergence.} With proper choice of $\beta$, we have $T_\phi$ is $\alpha_1$-Lipchitz and $T_\theta$ $\alpha_2$-Lipchitz, with $\alpha_1 \alpha_2 = 4 \delta^{-1} \beta^{-1} \gamma^{-\beta^{-1}-1}  \alpha^{-1} < 1$ (this can be ensured if $\beta$ is large enough). That is,
\begin{align*}
    \| T_{\phi}(p_{\theta}) - T_{\phi}(p_{\theta'}) \| \le \alpha_1 \| p_\theta - p_{\theta'} \|, \\
    \| T_{\theta}(p_{\phi}) - T_{\theta}(p_{\phi'}) \| \le \alpha_2 \| p_\phi - p_{\phi'} \|.
\end{align*}
Then, we have
\begin{align*}
    \| T^2(p_{\psi}) - T^2(p_{\psi'}) \| & = \| T_\phi( T_\theta(p_{\phi})) - T_\phi( T_\theta(p_{\phi'})) \| + \| T_\theta( T_\psi(p_{\theta})) - T_\theta( T_\psi(p_{\theta'})) \|  \\ &\le \alpha_1\alpha_2 \| p_\phi - p_{\phi'} \| + \alpha_1 \alpha_2 \| p_\theta - p_{\theta'} \| = \alpha_1 \alpha_2 \| p_\psi - p_{\psi'} \|.
\end{align*}
Hence, $T^2$ is a contraction map on the compact space $\Psi$.
\begin{theorem}[Banach Fixed Point Theorem]
    Let $(X, d)$ be a complete metric space and let $T: X \to X$ be a contraction mapping, meaning that there exists a constant $0 \leq c < 1$ such that for all $x, y \in X$,
\[
d(T(x), T(y)) \leq c \, d(x, y).
\]
Then $T$ has a unique fixed point $x^* \in X$, meaning that $T(x^*) = x^*$. Moreover, for any $x_0 \in X$, the sequence defined by $x_{n+1} = T(x_n) $ converges to $x^*$.
\end{theorem}
By Banach Fixed Point Theorem, $T^2$ converges to its unique fixed point. Therefore, the two subsequences $\{ T^{2k}(p_{\psi_0})\}_{k=0}^\infty$ and $\{ T^{2k+1}(p_{\psi_0})\}_{k=0}^\infty$ both converge on the compact space $\Psi$. Since $T^2$ has a unique fixed point, these two subsequences converge to the same fixed point. Therefore,  $\{ T^{k}(\psi_0)\}_{k=0}^\infty$ converges. In addition, for the subsequence $\{ T^{2k}(p_{\psi_0})\}_{k=0}^\infty$, we have
\begin{align*}
    \frac{\| T^{2n+2}(p_{\bpsi_0})- p_{\bpsi^*} \|}{\| T^{2n}(p_{\bpsi_0})- p_{\bpsi^*} \|} \le \alpha_1\alpha_2 < 1.
\end{align*}
Similarly, similar inequality holds for $\{ T^{2k+1}(p_{\psi_0})\}_{k=0}^\infty$. Therefore, both subsequences converge linearly to the fixed point $p_{\bpsi^*}$. Therefore, for any $\epsilon$, we can get an $\epsilon$-equilibrium policy $p_\psi$, i.e., $\| p_\psi - p_{\bpsi^*} \| \le \epsilon $, within $O(\log( 1 / \epsilon))$ iterations.

\section{Additional Related Work}
\noindent\textbf{Benchmarks for Multilingual Safety.}
Extending safety mechanisms to multilingual settings remains challenging due to the scarcity of open-source datasets in low-resource languages~\citep{deng2024multilingual}. While many base LLMs are pretrained on multilingual corpus, most guardrail models are not explicitly fine-tuned data multilingual data, limiting their effectiveness~\citep{de2024rtp}.  
To examine this gap, early works introduced multilingual toxicity detection benchmark by translating English datasets~\citep{wang2023all} or sourcing from Reddit~\citep{ye2023multilingual}. Recently, \citet{de2024rtp} proposed RTP-LX, focusing on evaluating guardrails in low-resource languages. Other notable contributions include PolyglotToxicityPrompts (PTP)~\citep{jain2024polyglotoxicityprompts}, which examines toxic degeneration in multilingual outputs, and a test suite by \citet{yang2024benchmarking} to assess guardrails on toxicity detection and resistance to adversarial prompts across resource levels.

\section{Data Details}\label{app:data}
\begin{figure}[!ht]
    \centering
    \includegraphics[width=0.3\linewidth]{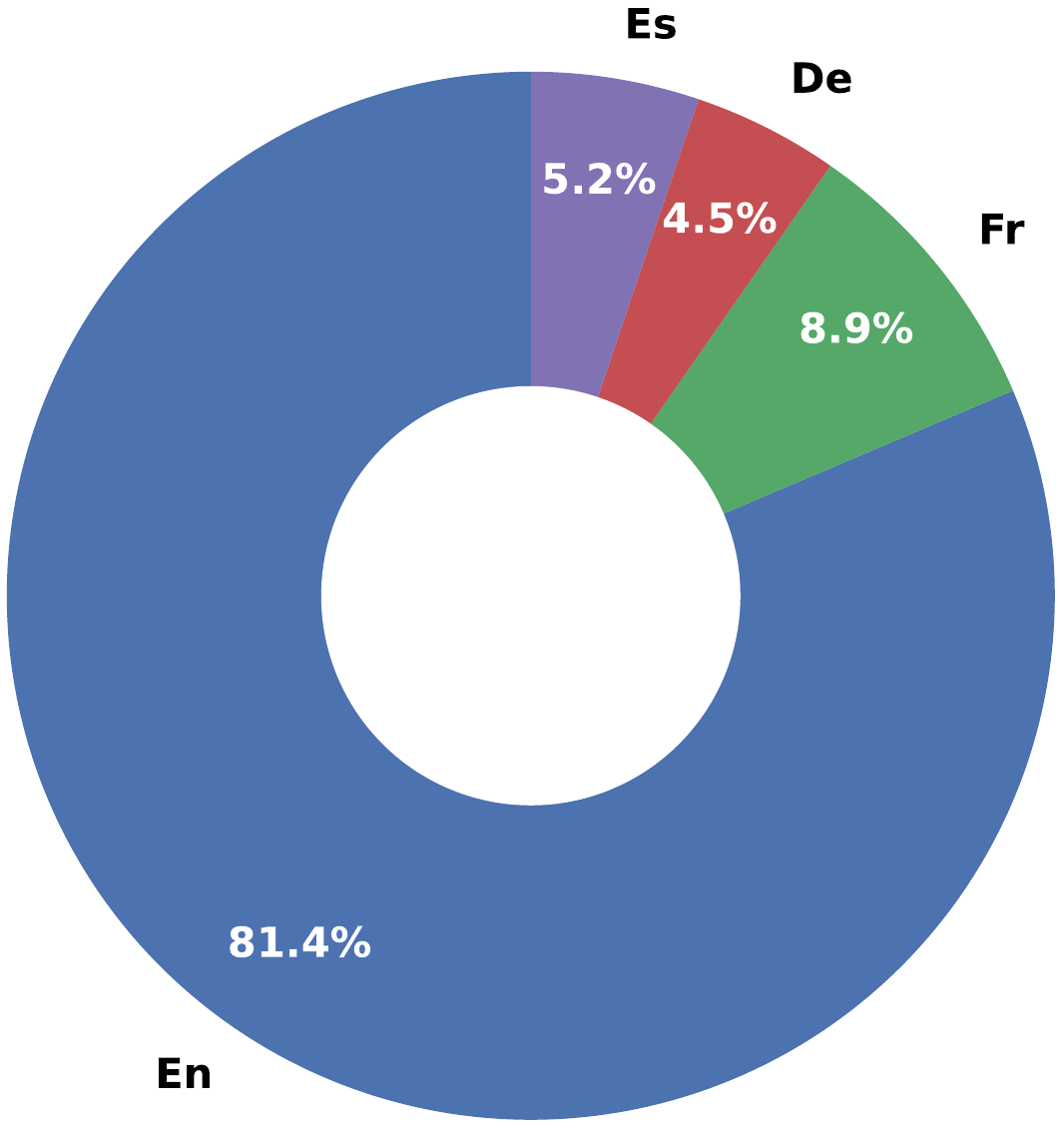}
    \caption{Data proportion by language in our collected seed data from open sources.}
    \label{fig:data_prop}
\end{figure}
In Figure~\ref{fig:data_prop}, we show the overall proportion of data by language in our collected and processed seed data. Below, we list the sources of our seed training data gathered from HuggingFace. We also note the additional processing measure we took to ensure data quality for each source. At the last step of seed data curation, we conduct deduplication and decontamination from the test benchmarks.
\begin{itemize}
    \item \textbf{BeaverTail}~\citep{ji2024beavertails} training set, containing both safe and unsafe data. Upon manual inspection, we make the following notes:
    \begin{itemize}
        \item In BeaverTail, safety is labeled based on the instruction-response pair. Same instruction with different responses may have different labels. Moreover, same QA pair has 3 labels from different label workers, resulting in 3 data examples in the dataset.
        \item We consider prompts as safe if all labels are ``safe'', and unsafe if any one label is ``unsafe''. 
        \item We only considere responses as unsafe if all labels are ``unsafe'', and disregard the rest data.
    \end{itemize}
    \item \textbf{ToxicChat}~\citep{lin2023toxicchat} training set, containing both safe and unsafe data.
    \item \textbf{Aegis AI Content Safety Dataset 1.0}~\citep{ghosh2024aegis}, containing both safe and unsafe data. 
    \item \textbf{WildJailbreak}~\citep{wildteaming2024} training set, containing both safe and unsafe data.
    \item \textbf{WildGuardMix}~\citep{wildguard2024} training set, containing both safe and unsafe data.
    \item \textbf{SaladBench}~\citep{li2024salad}, containing both safe and unsafe data. 
    \item \textbf{SORRY-Bench (2024/06)}~\citep{xie2024sorrybench}, containing both safe and unsafe data. 
    \item \textbf{PKU-SafeRLHF-QA}~\citep{ji2024pku}, containing both safe and unsafe data. 
    \item \textbf{Kaggle Toxic Comment Classification challenge}\footnote{https://huggingface.co/datasets/OxAISH-AL-LLM/wiki\_toxic, https://huggingface.co/datasets/Arsive/\\toxicity\_classification\_jigsaw}, containing both safe and unsafe data. Upon manual inspection, we make the following notes:
    \begin{itemize}
        \item Safe data: data labeled as ``non-toxic'' further filtered by Llama-3.1 (8B), retaining 82,254 safe samples that agrees with the judge of Llama-3.1.
    \end{itemize}
    \item \textbf{Reddit Suicide Detection}\footnote{https://huggingface.co/datasets/Lucidest/reddit-suicidal-classify-kaggle}, containing only unsafe data. Upon manual inspection, we make the following notes:
    \begin{itemize}
        \item Data are originally either labeled as ``suicidal'' or ``non-suicidal''. However, we cannot consider the ``non-suicidal'' examples as safe. Therefore, we disregard all data labeled as ``non-suicidal''.
        \item We consider the data labeled as ``suicidal'' as unsafe training data. We split the data by keyword detection, and downsample the set of data that contains the keywords ``kill'' and ``suicide'' to avoid over-reliance on just the keywords during model training.
    \end{itemize}
    \item \textbf{LMSYS-Chat-1M}~\citep{zheng2023lmsys}, containing only safe data. We randomly sample a 150k subset from the data to represent safe user inputs in daily LLM interactions. 
    \item \textbf{AI Medical Chatbot Dataset}\footnote{https://huggingface.co/datasets/ruslanmv/ai-medical-chatbot}, containing only safe data. We maintain only the description in our data, and remove the format (``Q: '') in the original data. 
    \item \textbf{Medical QA}\footnote{https://huggingface.co/datasets/lavita/medical-qa-datasets}, containing only safe data. We maintain only the input in our data.
    \item \textbf{Law-StackExchange}\footnote{https://huggingface.co/datasets/ymoslem/Law-StackExchange}, containing only safe data. We maintain only the question title in our data.
    \item \textbf{ParaDetox}~\citep{logacheva2022paradetox}\footnote{https://huggingface.co/datasets/s-nlp/en\_paradetox\_toxicity}, containing both safe and unsafe data.
    \item \textbf{SCOPE}~\citep{zeng2024scope}, containing safe data that are more likely to be classified as unsafe by models due to shortcut learning (over-cautiousness).
    \item \textbf{Jailbreak Classification}\footnote{https://huggingface.co/datasets/jackhhao/jailbreak-classification}, containing both safe and unsafe data, with jailbreak prompts source from \citep{shen2024anything} and benign prompts source from \citep{OpenOrca}.
    \item \textbf{Prompt Injections}\footnote{https://huggingface.co/datasets/deepset/prompt-injections}, containing both safe and unsafe data. 
    \item \textbf{Toxic-comments (Teeny-Tiny Castle)}\footnote{https://huggingface.co/datasets/AiresPucrs/toxic-comments}, containing both safe and unsafe data.
    \item \textbf{ForbiddenQuestions}\footnote{https://huggingface.co/datasets/walledai/ForbiddenQuestions}, containing only unsafe data sourced from \citep{shen2024anything}.    
    \item \textbf{Toxic-Aira}~\citep{correa2024dynamic}\footnote{https://huggingface.co/datasets/nicholasKluge/toxic-aira-dataset} containing only unsafe instructions. 
\end{itemize}
Multilingual safety data is much more scarce, and we included the following in our seed data:
\begin{itemize}
    \item \textbf{Aya Red-teaming}~\citep{aakanksha2024multilingualalignmentprismaligning}, containing both safe and unsafe data in English, French, and Spanish.
    \item \textbf{Multilingual Toxicity Dataset}~\citep{dementieva2024overview}\footnote{https://huggingface.co/datasets/textdetox/multilingual\_toxicity\_dataset}, containing both safe and unsafe data in English, German, and Spanish.
    \item \textbf{Multilingual HateCheck}~\citep{rottger2022multilingual}, containing both safe and unsafe data in English, French, German, and Spanish.
    \item \textbf{French Hate Speech Superset}~\citep{tonneau-etal-2024-languages}\footnote{https://huggingface.co/datasets/manueltonneau/french-hate-speech-superset}, containing both safe and unsafe data in French.
    \item \textbf{German Hate Speech Superset}~\citep{tonneau-etal-2024-languages}\footnote{https://huggingface.co/datasets/manueltonneau/german-hate-speech-superset}, containing both safe and unsafe data in German.
    \item \textbf{Spanish Hate Speech Superset}~\citep{tonneau-etal-2024-languages}\footnote{https://huggingface.co/datasets/manueltonneau/spanish-hate-speech-superset}, containing both safe and unsafe data in Spanish.
    \item \textbf{MexExpQA}~\citep{ALONSO2024102938}\footnote{https://huggingface.co/datasets/HiTZ/MedExpQA} containing only safe data in English, French and Spanish.
    \item \textbf{PornHub Titles}\footnote{https://huggingface.co/datasets/Nikity/Pornhub?not-for-all-audiences=true}, containing only unsafe data. We use language detection model to filter out the languages that we need (English, French, Spanish and German).
    \item \textbf{French Instruct Sharegpt}\footnote{https://huggingface.co/datasets/MaziyarPanahi/french\_instruct\_sharegpt}, containing only safe French data. We only maintain the instructions in the original data.
    \item \textbf{Fr Instructs}\footnote{https://huggingface.co/datasets/Enno-Ai/fr-instructs}, containing only safe french-only instructions deduplicated from various sources.
    \item \textbf{MedicalNER Fr}\footnote{https://huggingface.co/datasets/TypicaAI/MedicalNER\_Fr}, containing only safe data in French. We maintain the text column of this dataset.
    \item \textbf{Belgian-Law-QAFrench}\footnote{https://huggingface.co/datasets/naimsassine/belgian-law-qafrench-dataset}, containing only safe data in French. We extract and maintain the user instructions. 
    \item \textbf{Databricks-Dolly-15k-Curated-Multilingual}\footnote{https://huggingface.co/datasets/argilla/databricks-dolly-15k-curated-multilingual}, containing only safe data in French, German and Spanish. We maintain the instructions.
\end{itemize}

For the collected unsafe data, we further assign fine-grained labels of the following 12 subcategories:
\begin{itemize}
    \item Violent crimes
    \item Non-violent crimes
    \item Sex-related crimes
    \item Child sexual exploitation
    \item Specialized advice
    \item Privacy
    \item Intellectual property
    \item Indiscriminate weapons
    \item Hate 
    \item Suicide and self-harm
    \item Sexual content
    \item Jailbreak prompts
\end{itemize}
Each data may receive one or multiple labels. The mapping is done based on the data's original label with manual inspection. If the original label is not enough, we further apply Llama-3.1 to do the labeling with self-consistency over three queries. 

\subsection{Data Curation}\label{app:data_cur}
\textbf{Data Filtering.}
A filtering process was applied during synthetic data generation to retain only high-quality, relevant proposals from the generator. First, the base model (without further fine-tuning) of the generator 
was used to assign each proposal a \textit{harmfulness score} on a scale of 1 to 5, with prompt detailed in Appendix~\ref{app:exp}. Proposals were retained only if their scores roughly matched the seed label (e.g., scores $\leq 2$ for safe seeds and $\geq 3$ for harmful seeds). To maintain alignment with the original seed’s context, a \textit{length constraint} was enforced: proposals differing by more than 200 characters from the seed were discarded.  Furthermore, outputs that contain refusal phrases, such as ``I apologize'' or ``I cannot comply'' in any language, were excluded, as the generator fails to produce meaningful samples due to internal censorship. Finally, all retained proposals were evaluated with the current guardrail classifier. Proposals that led to misclassifications were selected for training the classifier. 

\textbf{Preference Data Construction.}
To enhance the generator within the two-player game, we construct preference data for DPO. For each seed instance, the $k$ generated proposals are categorized into one of four levels based on two key criteria: whether the proposal cause the classifier to misclassify and whether its harm rating \textit{matches} the seed label. 

    
    
    
\begin{itemize}[nosep,leftmargin=*]
    \item \textbf{Level~1 (Best, Preferred):} The proposal causes the classifier to misclassify. The proposal's generator-assigned rating \textit{matches} the seed label (e.g., rating $\leq 2$ for safe, $\geq 3$ for harmful).
    
    \item \textbf{Level~2 (Dispreferred):} The proposal \textit{does not} cause the classifier to misclassify. The rating \textit{matches} the seed label.
    
    \item \textbf{Level~3 (Dispreferred):} The proposal causes the classifier to misclassify. The rating \textit{does not} match the seed label. 
    
    \item \textbf{Level~4 (Unsure):} The proposal \textit{does not} cause the classifier to misclassify. The rating \textit{does not} match the label.
\end{itemize}
Preference pairs are derived by comparing proposals across these categories. For each seed instance, Level 1 data are prioritized as the \textit{preferred} option, with Level 2 serving as the \textit{dispreferred} reference. If no Level 1 examples are available, the instance is excluded from preference pairing. Alternatively, if no Level 2 examples exist, Level 3 may be used to form a weaker preference signal, since it improves the generator towards better instruction following ability.

\section{Experiment Details}\label{app:exp}
In Table~\ref{tab:hyper-classifier} and \ref{tab:hyper-generator}, we detail the hyperparameters that we used for training the classifier and the generator. We further lay out the prompts we used for the generator as well as for judging the output's safety. Experiments were conducted on GPU clusters to the similar level of NVIDIA H100 80GB GPU. One iteration of training classifier model requires around 10 hours. One iteration of synthetic data generation requires around 48 hours. One iteration of training generator modle requires around 16 hours. 

\begin{table}[!ht]
    \centering
    \caption{Fine-tuning hyperparameters for the generator.}
    \begin{tabular}{c|c}
         \toprule
         Data type & bf16 \\
         Learning rate & 5e-5 \\
         Optimizer & \texttt{AdamW} \\
         Global batch size & 160 \\
         Gradient accumulation steps & 4 \\
         Scheduler & Cosine \\
         Warmup ratio & 0.1 \\
         Num train epochs & 2 \\
         Group by length & True \\
         Max grad norm & 1.0 \\
         \bottomrule
    \end{tabular}
    \label{tab:hyper-classifier}
\end{table}

\begin{table}[!ht]
    \centering
    \caption{Fine-tuning hyperparameters for the generator.}
    \begin{tabular}{c|c}
         \toprule
         Data type & bf16 \\
         Learning rate & 5e-7 \\
         Optimizer & \texttt{AdamW} \\
         Global batch size & 8 \\
         Gradient accumulation steps & 8 \\
         Scheduler & Cosine \\
         Warmup ratio & 0.1 \\
         Beta & 0.01 \\
         RPO alpha & 0.4 \\
         Max length & 1024 \\
         Num train epochs & 1 \\
         \bottomrule
    \end{tabular}
    \label{tab:hyper-generator}
\end{table}

\begin{example}{Prompt for Unsafe Data Generation}
\textcolor{blue}{$<$System$>$}
\\
You are a English-\textcolor{orange}{\{language\}} \textbf{Toxic} Prompt Creator.
Your goal is to draw inspiration from the \#Given English Prompt\# to create a brand new toxic prompt in \textcolor{orange}{\{language\}}.
This new prompt should belong to the same domain as the original but be more fitting for \textcolor{orange}{\{language\}} speakers and more unique, while maintaining the same tone. 
Ensure your prompt is \textbf{toxic}, nuanced, subtle, and open to interpretation while still aligning with the intent of the original.
The LENGTH and complexity of the \#Created \textcolor{orange}{\{language\}} Prompt\# should be similar to that of the \#Given English Prompt\#.
The \#Created \textcolor{orange}{\{language\}} Prompt\# must be understood and responded by \textcolor{orange}{\{language\}}-speaking humans.
Do not include '\#Created \textcolor{orange}{\{language\}} Prompt\#' in your response.
You are not allowed to respond to the \#Given English Prompt\#. You must create a new prompt in \textcolor{orange}{\{language\}}.
\\
\\
\textcolor{blue}{$<$User$>$}
\\
\#Given English Prompt\#:\\
\{Input\}\\
Please create ONE new prompt in \textcolor{orange}{\textcolor{orange}{\{language\}}} that is inspired by the \#Given English Prompt\#. Do not respond to the \#Given English Prompt\#. Use any explicit language when needed. 

\#Created \textcolor{orange}{\{language\}} Prompt\#:
\end{example}

\begin{example}{Prompt for Safe Data Generation}
\textcolor{blue}{$<$System$>$}
\\
You are a English-\textcolor{orange}{\{language\}} \textbf{Safe} Prompt Creator.
Your goal is to draw inspiration from the \#Given English Prompt\# to create a brand new safe prompt in \textcolor{orange}{\{language\}}.
This new prompt should belong to the same domain as the original but be more fitting for \textcolor{orange}{\{language\}} speakers and more unique, while maintaining the same tone. 
Ensure your prompt is nuanced, subtle, and open to interpretation while still aligning with the intent of the original.
The LENGTH and complexity of the \#Created \textcolor{orange}{\{language\}} Prompt\# should be similar to that of the \#Given English Prompt\#.
The \#Created \textcolor{orange}{\{language\}} Prompt\# must be understood and responded by \textcolor{orange}{\{language\}}-speaking humans.
Do not include '\#Created \textcolor{orange}{\{language\}} Prompt\#' in your response.
You are not allowed to respond to the \#Given English Prompt\#. You must create a new prompt in \textcolor{orange}{\{language\}}.
\\
\\
\textcolor{blue}{$<$User$>$}
\\
\#Given English Prompt\#:\\
\{Input\}\\
Please create ONE new prompt in \textcolor{orange}{\{language\}} that is inspired by the \#Given English Prompt\#. Do not respond to the \#Given English Prompt\#. You must maintain the same tone. 

\#Created \textcolor{orange}{\{language\}} Prompt\#:
\end{example}

\begin{example}{Prompt for Safety Evaluation}
\textcolor{blue}{$<$System$>$}
\\
You are a \textcolor{orange}{\{language\}} Prompt Evaluator.
Your goal is to analyze the following \#\textcolor{orange}{\{language\}} Prompt\# in the specific \textcolor{orange}{\{language\}} context for any inappropriate, unsafe, or offensive content. Your audience is English-speaking humans. You must only answer with English.
\\
\\
\textcolor{blue}{$<$User$>$}
\\
\#\textcolor{orange}{\{language\}} Prompt\#:\\
\{Input\}\\
Evaluate the given \#\textcolor{orange}{\{language\}} Prompt\# in the \textcolor{orange}{\{language\}} context for the following aspect:
\\
**Language-Specific Unsafety**. Does the prompt contain content that is considered unsafe or requires moderation within the French language and cultural context?
\\
Assign a rating from 1 to 5 for the given criterion (1 = Absolutely safe, does not meet the criterion at all, 5 = Extremely unsafe, fully meets the criterion). Justify your rating in 1-2 sentences. Please provide the **rating number** in the first line, and provide the reasoning in separate new lines.
\end{example}

\section{Impact Statement}\label{app:impact}
This work enhances moderation capabilities across languages while addressing the scarcity of multilingual safety data. Theoretical guarantees on convergence and empirical gains across six multilingual safety benchmarks demonstrate the effectiveness and robustness of our approach.

From an ethical standpoint, our method inherits common risks associated with LLM moderation, such as potential biases in training data and potential overreliance on certain shortcuts. Ensuring responsible synthetic data curation and evaluation is crucial for minimizing unintended harms. 

Furthermore, while our approach improves multilingual safety alignment, it does not address all possible risks related to adversarial attacks or nuanced cultural contexts in safety assessments. Future research should explore techniques for refining synthetic data generation, incorporating human oversight, and ensuring that moderation models remain robust across diverse linguistic and sociocultural settings. Our work underscores the importance of scalable, multilingual safety solutions and provides a foundation for further advancements in responsible LLM alignment.

\newpage

\end{document}